%% file: main.tex
\documentclass[conference]{IEEEtran}
\usepackage{times}

\usepackage[numbers]{natbib}
\usepackage{multicol}
\usepackage[bookmarks=true]{hyperref}

\usepackage{graphicx}
\usepackage{amsmath,amsfonts}
\usepackage{array}
\usepackage{textcomp}
\usepackage{stfloats}
\usepackage{url}
\usepackage{verbatim}
\usepackage{graphicx}
\usepackage{balance}

\usepackage{algorithm}
\usepackage{algpseudocode}
\algrenewcommand\algorithmicrequire{\textbf{Input:}}
\algrenewcommand\algorithmicensure{\textbf{Output:}}

\usepackage{tikz}
\usepackage{pgfplots}
\usepackage{xcolor}
\pgfplotsset{compat=1.18}
\usetikzlibrary{positioning}
\usetikzlibrary{decorations.pathreplacing}
\usetikzlibrary{calc}
\usetikzlibrary{positioning,shapes,shadows,arrows}

\usepackage{pifont}
\usepackage{multirow}
\usepackage{dsfont}
\usepackage{subcaption}
\usepackage{bm}
\usepackage{caption}
\usepackage{bbm}
\usepackage{gensymb}
\usepackage{amsthm}
\usepackage{amsmath, amssymb, mathtools}

\usepackage{array} 
\newcolumntype{P}[1]{>{\centering\arraybackslash}p{#1}}

\pgfplotsset{
every axis/.append style={
  axis line style={->}, 
  legend style={font=\scriptsize},
  label style={font=\scriptsize},
  title style={font=\scriptsize},
  tick label style={font=\scriptsize},
  }
}

\usepackage[switch]{lineno}

\newtheorem{theorem}{Theorem}[section]
\newtheorem{definition}{Definition}[section]
\newtheorem{proposition}{Proposition}[section]
\newtheorem{lemma}{Lemma}[section]
\newtheorem{assumption}{Assumption}[section]

\begin{document}

\title{
Zero-Shot Sim-to-Real Robot Learning: A Dexterous Manipulation Study on Reactive Catching
}

\makeatletter
\def\authorrefmark#1{\textsuperscript{#1}}
\makeatother

\author{\authorblockN{Kejia Ren\authorrefmark{1},
Gaotian Wang\authorrefmark{1},
Andrew S. Morgan\authorrefmark{2} and
Kaiyu Hang\authorrefmark{1}}
\authorblockA{\authorrefmark{1}Department of Computer Science, Rice University}
\authorblockA{\authorrefmark{2}Robotics and AI Institute}}

\maketitle

\input{includes/abstract}

\IEEEpeerreviewmaketitle

\input{includes/intro}
\input{includes/relatedwork}
\input{includes/method}

\input{includes/problem}

\input{includes/expsim}
\input{includes/expreal}
\input{includes/conclusion}

\section*{Acknowledgments}

This work was supported by the U.S. National Science Foundation (NSF) under grant FRR-2240040.

\bibliographystyle{plainnat}
\bibliography{refs}

\input{includes/appendix_alg}
\input{includes/appendix}

\end{document}

%% file: includes/abstract.tex
\begin{abstract}
Dexterous manipulation is physics-intensive and highly sensitive to modeling errors and perception noise, making sim-to-real transfer prohibitively challenging.
Domain randomization (DR) is commonly used to improve the robustness of learned policies for such tasks,
but conventional DR randomizes one instance per episode,
offering very limited exposure to the variability of real-world dynamics.
To this end,
we propose Domain-Randomized Instance Set (DRIS), which represents and propagates a set of randomized instances simultaneously,
providing richer approximation of uncertain dynamics and enabling policies to learn actions that account for multiple possible outcomes.
Supported by theoretical analysis,
we show that DRIS yields more robust policies and alleviates the need for real-world fine-tuning,
even with a modest number of instances (e.g., 10).
We demonstrate this on a challenging reactive catching task.
Unlike traditional catching setups that use end-effectors designed to mechanically stabilize the object (e.g., curved or enclosing surfaces), 
our system uses a flat plate that offers no passive stabilization, 
making the task highly sensitive to noise and requiring rapid reactive motions.
The learned policies exhibit strong robustness to uncertainties and achieve reliable zero-shot sim-to-real transfer.
\end{abstract}

%% file: includes/intro.tex
\section{Introduction}
\label{sec:intro}

Dexterous manipulation refers to the robot's ability to skillfully manipulate objects through coordinated and intentional contact interactions,
e.g., in-hand manipulation~\cite{andrychowicz2020learning, chen2022system}, 
tool use~\cite{liu2023learning},
pushing~\cite{dengler2022learning}, 
catching~\cite{kim2014catching}, etc.
Such tasks often require well-planned contact, timing, and motion to achieve stable and adaptive control over object behavior.
Moreover, they often introduce a substantial sim-to-real gap,
as the success of such tasks can be highly sensitive to inaccuracies in estimating the underlying dynamics and physical properties such as contact geometry, frictions, object inertia, and compliance.
Even small discrepancies between simulation and the real world can lead to large deviations in manipulation behavior and ultimately cause task failure.

Traditional model-based frameworks depend on known system parameters to achieve accurate prediction and control~\cite{chavan2015prehensile}.
However, in practical manipulation scenarios, 
properties such as contact stiffness, friction coefficients, and object geometries can greatly vary between environments and are difficult to precisely identify.
Learning-based methods address these limitations by training policies with extensive domain randomization or data augmentation over uncertain physical parameters in simulation~\cite{peng2018sim}, improving the robustness to these uncertainties when deployed in the real world.
Yet, tasks involving very dynamic, contact-rich interactions remain difficult to learn, as their success hinges on precise coordination of contact timing and force modulation under uncertainty.

\begin{figure}[t]
    \centering
    \includegraphics[width=\columnwidth]{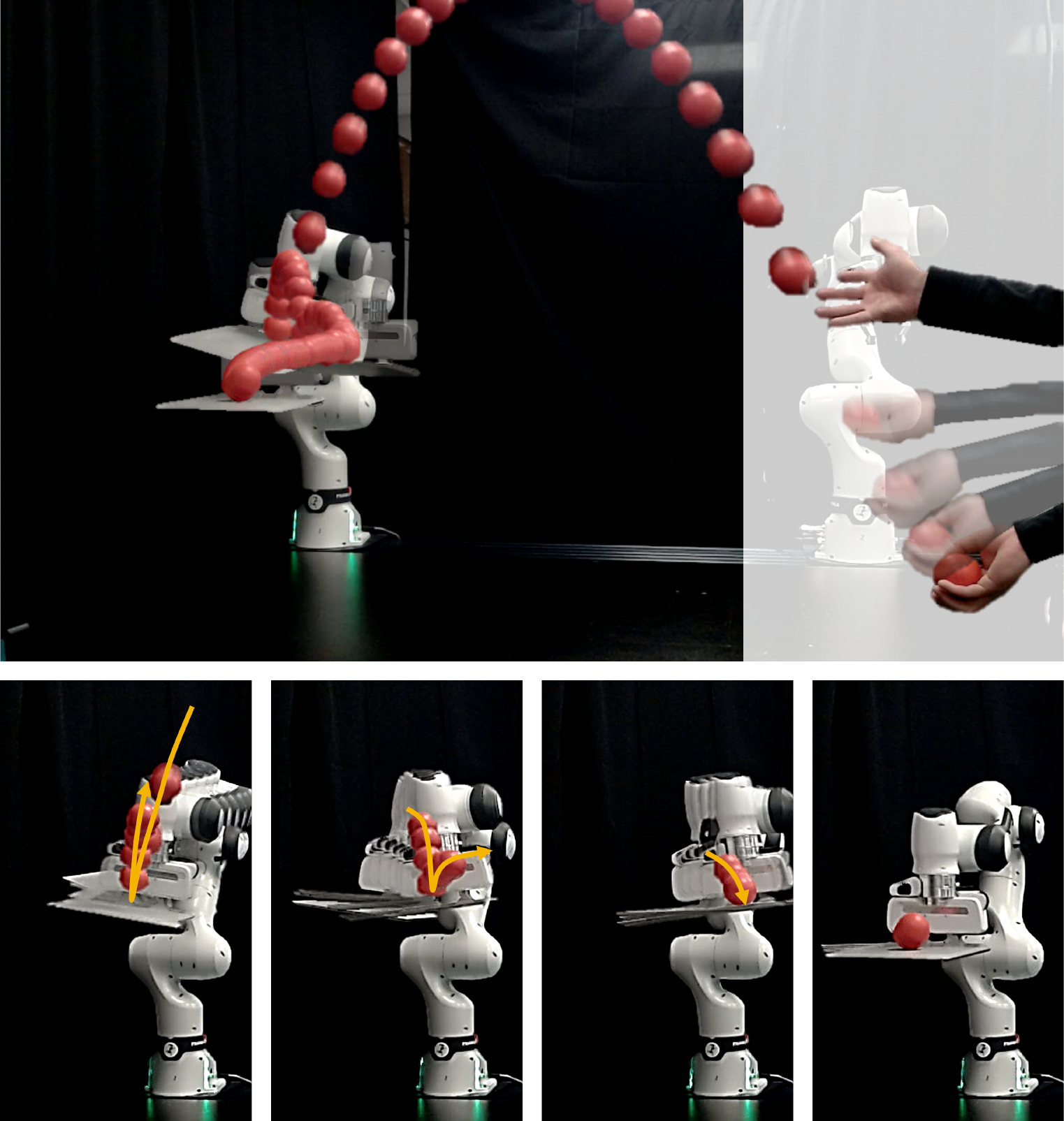}
    \caption{
        Robot reactive catching of a rubber ball using a flat plate.
        The second row shows sequential impacts during the catching motion, where the final frame depicts the ball successfully stabilized on the plate;
        the yellow arrows indicate the ball's trajectory at the moments of impact.
    }
    \label{fig:first_pic}
\end{figure}

Domain randomization is traditionally performed by perturbing relevant physical parameters during data collection,
e.g., sampling one object instance with randomized properties in each rollout and training policy on the collected rollouts.
While this exposes the policy to diverse experiences,
it lacks a structured mechanism for the policy to reason about how uncertainties affect possible manipulation outcomes.
In contrast, if such uncertainties are explicitly incorporated into the task representation itself,
the policy can learn to generate actions that compensate for them and plan with respect to a distribution of possible outcomes,
thereby mitigating the sim-to-real gap introduced by these uncertainties.

To this end,
rather than training a policy on individual manipulation instances with randomly sampled physical parameters (as is done in conventional domain randomization), 
we propose \emph{Domain-Randomized Instance Set (DRIS)} that explicitly models a distribution of possible physical variations encountered during manipulation.
Our proposed DRIS contains multiple parallel instances (e.g., objects with different physical properties), all of which are propagated simultaneously.
Their state evolutions under a shared action jointly inform the policy update.
As a result, the learned policy becomes inherently tolerant to the uncertainties encoded in DRIS, leading to improved robustness against the sim-to-real gap (e.g., enabling reliable zero-shot transfer).

To demonstrate the effectiveness of the proposed learning strategy,
we consider a highly challenging reactive catching task illustrated in Fig.~\ref{fig:first_pic}.
Unlike traditional catching setups that employ curved or concave end-effectors such as cups~\cite{dong2020catch}, nets~\cite{gold2022catching}, or articulated hands~\cite{bauml2010kinematically}, to mechanically stabilize the object after impact,
our setup uses a flat and low-friction plate rigidly attached to the robot's end-effector.
As a result,
the system becomes highly sensitive to physical uncertainties or perception errors,
where even small errors in contact timing, plate orientation, or unwanted lateral forces can cause the ball to bounce off or roll away from the plate.

In this work, we propose a robust learning strategy for dynamic manipulation,
demonstrated through a challenging robot reactive catching case study.
Our key technical contributions are summarized as follows:
\begin{enumerate}

    \item We propose Domain-Randomized Instance Set (DRIS) as a representation to capture the uncertainty in system dynamics induced by physical parameter variations;

    \item We propose a DRIS-based data collection and policy learning paradigm that enables simultaneous handling of multiple task instances, yielding more stable and robust policy learning and improved zero-shot sim-to-real transfer than the traditional paradigm, supported by theoretical analysis;

    \item We instantiate the framework on a challenging reactive catching problem with a tailored task representation and policy architecture, and validate our framework through both simulation and real-world experiments.
\end{enumerate}

%% file: includes/relatedwork.tex
\section{Related Work}
\label{sec:relatedwork}

\emph{Reinforcement Learning for Dexterous Manipulation.} 
Unlike traditional analysis-based approaches sensitive to modeling errors, 
Reinforcement Learning (RL) learns directly from data, 
making it advantageous for the complex contact dynamics of dexterous manipulation. 
RL has successfully solved tasks including in-hand manipulation~\cite{andrychowicz2020learning, chen2022system}, catching~\cite{abeyruwan2023agile}, juggling~\cite{ploeger2021high}, deformable object manipulation~\cite{chi2024iterative}, and regrasping~\cite{petrenko2023dexpbt}. 
However, these data-intensive methods typically require large-scale high-fidelity simulation or expert demonstrations to facilitate real-world transfer.

\emph{Sim-to-Real Transfer in Robotics.} 
Bridging the sim-to-real gap caused by sensing and contact discrepancies is a central challenge. Beyond Domain Randomization (DR), researchers utilize high-fidelity simulators~\cite{todorov2012mujoco, sanchez2020learning, makoviychuk2021isaac} or explicitly bridge the gap via adaptive system identification~\cite{yu2017preparing, jatavallabhula2021gradsim, sobanbabu2025sampling} and hybrid dynamics learning~\cite{ajay2018augmenting, jiang2021simgan}. 
While these techniques enable zero-shot transfer~\cite{tobin2017domain}, reliable deployment for dynamic manipulation remains difficult due to the discontinuous, stochastic nature of high-speed contact dynamics.

\emph{Domain Randomization in Robot Learning.} 
Since real-world data collection is costly, 
DR is essential for reducing the simulation gap in contact-rich tasks~\cite{muratore2022robot, huber2024domain}. 
Randomization strategies have evolved from visual and geometric properties~\cite{tobin2017domain, tobin2018domain} to kinematic~\cite{exarchos2021policy} and dynamic parameters~\cite{peng2018sim}. 
While active DR optimizes sampling distributions to improve transfer~\cite{ramos2019bayessim, chebotar2019closing, mehta2020active, muratore2022neural, tiboni2023dropo}, 
it typically requires real-world rollouts. 
Although some established works strategically employ DR—either through curriculum-based scaling~\cite{akkaya2019solving} or entropy maximization~\cite{tiboni2024domain}—to enable zero-shot ransfer,
our approach differs by training over the joint evolution of multiple randomized instances simultaneously to achieve zero-shot transfer for challenging tasks such as catching.

\emph{Representing Uncertainty in States and Dynamics.} 
Robotics has extensively studied uncertainty representation. 
Classical methods employ belief-space planning via POMDPs~\cite{porta2006point, bry2011rapidly, jain2018efficient}. 
Learning-based approaches incorporate uncertainty through probabilistic latent dynamics~\cite{hafner2019dream, hafner2019learning}, history-based parameter inference~\cite{nagabandi2018learning, sheckells2019using}, or ensembles for model-based RL~\cite{chua2018deep}.
While some history-based approaches require real-world data in the loop~\cite{sheckells2019using}, our approach is purely zero-shot.
Furthermore, while others can perform adaptation online~\cite{nagabandi2018learning}, 
they are not suitable enough for highly dynamic tasks such as catching, where the robot must react within milliseconds.
Complementary to these, recent work learns belief embeddings to explicitly encode uncertainty for policy conditioning in partially observable settings~\cite{wang2023learning}.
While such belief-representation methods learn a latent state space to implicitly represent uncertainty, 
our approach more explicitly and intuitively captures the uncertainty through a multi-instance representation during high-speed contacts.

\emph{Dexterous Robot Catching.} 
Catching typically requires fast prediction and reactive motion. 
Prior work evolved from optimization-based planning~\cite{bauml2010kinematically} and probabilistic models for uneven objects~\cite{kim2014catching} to deep learning-based policies on mobile manipulators~\cite{zhang2025catch}. 
However, most rely on mechanically assisting end-effectors like cups~\cite{namiki2014ball, dong2020catch, abeyruwan2023agile}, nets~\cite{gold2022catching}, or articulated hands~\cite{salehian2016dynamical, huang2023dynamic}. 
In contrast, we use a flat plate without mechanical stabilization. 
Unlike similar flat-surface work~\cite{batz2010dynamic} limited to specific objects and a number of rebounds, our approach handles diverse conditions without such restrictions.

%% file: includes/method.tex
\section{Learning Manipulation Policy with Domain-Randomized Instance Set}
\label{sec:method}

\subsection{Dexterous Manipulation as Policy Learning}
We consider the dexterous manipulation problem in which a robotic manipulator interacts with an object through contact-rich dynamics.
We define the domain space $\mathcal{C}$ as the set of all relevant physical parameters that influence the system dynamics, such as robot characteristics (e.g., stiffness), object geometry and physical properties (e.g., friction).
Each $\bm{c} \in \mathcal{C}$ represents a specific domain instance 
that induces a particular realization of the system dynamics.
These parameters are often difficult to measure accurately, but they can substantially affect the object's motion during manipulation.

The problem is formulated as a discrete-time Markov Decision Process (MDP).
At time step $t$, the system observes a state $\bm{s}_t \in \mathcal{S}$, executes an action $\bm{a}_t = \pi(\bm{s}_t) \in \mathcal{A}$ according to a robot policy $\pi(\cdot)$, 
and transitions to the next state $\bm{s}_{t+1} = f \left(\bm{s}_t, \, \bm{a}_t, \, \bm{c} \right)$, where the dynamics $f$ depend on the domain parameter $\bm{c}$.
Different values of $\bm{c}$ (e.g., object properties) lead to distinct motion behaviors even under the same robot action.
At each step, the system receives a scalar reward $r_t = r\left(\bm{s}_t, \, \bm{a}_t \right) \in \mathbb{R}$.

The objective is to find an optimal policy $\pi_{\bm{\theta}}: \mathcal{S} \mapsto \mathcal{A}$, that maximizes the expected return (i.e., cumulative reward), where $\gamma \in [0, 1)$ is a discount factor:
\begin{equation}
\begin{aligned}
    \max_{\pi_{\bm{\theta}}} \quad & \mathbb{E}_{\pi_{\bm{\theta}}}\!\left[ \sum_{t=0}^{T-1} \gamma^t \, r\left(\bm{s}_t, \, \bm{a}_t \right) \right]\\
	\text{s.t.}
    \quad & \bm{a}_t = \pi_{\bm{\theta}}\left(\bm{s}_t \right) \in \mathcal{A}, \quad \forall t = 0, \cdots, T - 1, \\
    \quad & \bm{s}_{t+1} = f \left(\bm{s}_t, \, \bm{a}_t, \, \bm{c} \right) \in \mathcal{S}, \quad \bm{c} \in \mathcal{C}
\end{aligned}
\end{equation}
In deep reinforcement learning (DRL),
the policy $\pi_{\bm{\theta}}$ is modeled by a neural network parameterized by $\bm{\theta}$, 
and optimized via gradient-based update using data collected from robot-object interactions.

\subsection{Domain-Randomized Instance Set (DRIS)}
\label{sec:dris}

In practice, physical parameters $\bm{c}$ are difficult to identify accurately, introducing uncertainty into manipulation dynamics. 
To improve robustness, most approaches model the dynamics stochastically as a probability $P\left(\bm{s}_{t+1} | \bm{s}_t, \, \bm{a}_t\right)$ and apply domain randomization to approximately capture this stochasticity,
typically by sampling a single parameter instance per episode in standard model-free RL~\cite{peng2018sim}. 
More recent methods further adapt the sampling distribution using real-world data~\cite{ramos2019bayessim,tiboni2023dropo}. 
However, these strategies still predict only individual next states, and do not explicitly model how the distribution of possible states evolves over time.

Inspired by prior work on set-based representations~\cite{wang2025caging},
rather than propagating a single state $\bm{s}_t$,
we instead propagate a set of possible states $\mathcal{S}_t \subset \mathcal{S}$ simultaneously under a shared robot action,
reflecting the effects of physical variations across different domain parameters.
Specifically, we construct a discrete set $\hat{\mathcal{C}} \subset \mathcal{C}$ of $N$ manipulation instances (indexed by $i$) by uniformly sampling from the domain space:
\begin{equation}
    \hat{\mathcal{C}} = \left\{ \bm{c}^{(i)} \right\}_{i=1}^N, \,\, \bm{c}^{(i)} \sim \mathcal{U}\left( \mathcal{C}\right)
\end{equation}
We then define the \emph{Domain-Randomized Instance Set (DRIS)} as the collection of state-parameter pairs associated with these sampled instances:
\begin{equation}
    \mathcal{D}_t = \left\{ \left(\bm{s}^{(i)}_t, \bm{c}^{(i)} \right) \,\middle|\, \bm{c}^{(i)} \in \hat{\mathcal{C}} \right\}_{i=1}^N \subset \mathcal{S} \times \mathcal{C} 
\end{equation}
where $\bm{s}_t^{(i)}$ in each element denotes the state of the $i$-th manipulation instance, as it evolves at time $t$ under its corresponding domain parameter $\bm{c}^{(i)}$.
For clarity, we also extract the state of all instances in the DRIS to form another set (i.e., the projection of $\mathcal{D}_t$ onto the state space $\mathcal{S}$):
\begin{equation}
    \mathcal{S}_t = \text{proj}_{\mathcal{S}} \left(\mathcal{D}_t\right) = \left\{ \bm{s}_t^{(i)} \,\middle|\, \left( \bm{s}_t^{(i)}, \, \bm{c}^{(i)} \right) \in \mathcal{D}_t \right\} \subset \mathcal{S}
\end{equation}
 We extend the system dynamics by introducing a new function $\mathcal{F}$ to evolve the entire instance set:
\begin{equation}
\begin{aligned}
    \mathcal{D}_{t+1} &= \left\{ \left( \bm{s}_{t+1}^{(i)}, \, \bm{c}^{(i)} \right) \right\}_{i=1}^N = \mathcal{F} \left(\mathcal{D}_t, \, \bm{a}_t \right)\\
    & = \left\{ \left(f\left(\bm{s}^{(i)}_t, \, \bm{a}_t, \, \bm{c}^{(i)} \right), \, \bm{c}^{(i)}\right) \, \middle| \, \left(\bm{s}^{(i)}_t, \, \bm{c}^{(i)} \right) \in \mathcal{D}_t\right\}
\end{aligned}
\label{eq:dris_prop}
\end{equation}

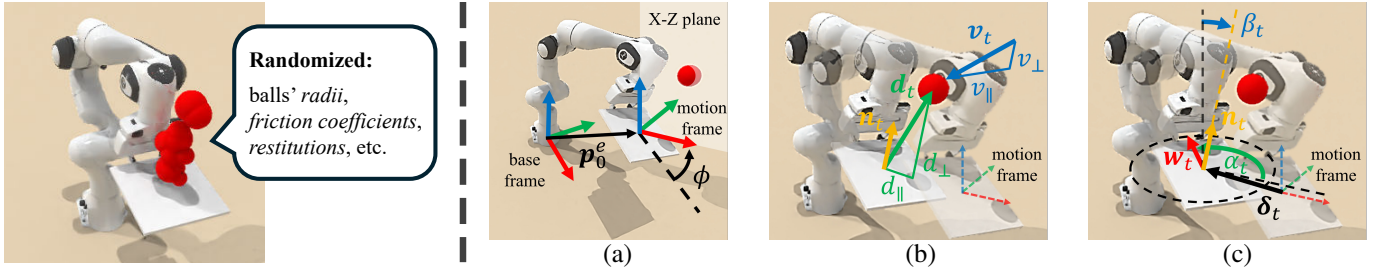
\begin{figure*}[t]
    \centering
     \begin{tikzpicture}
        \node[anchor=south west,inner sep=0] at (0,0){\includegraphics[width=\linewidth]{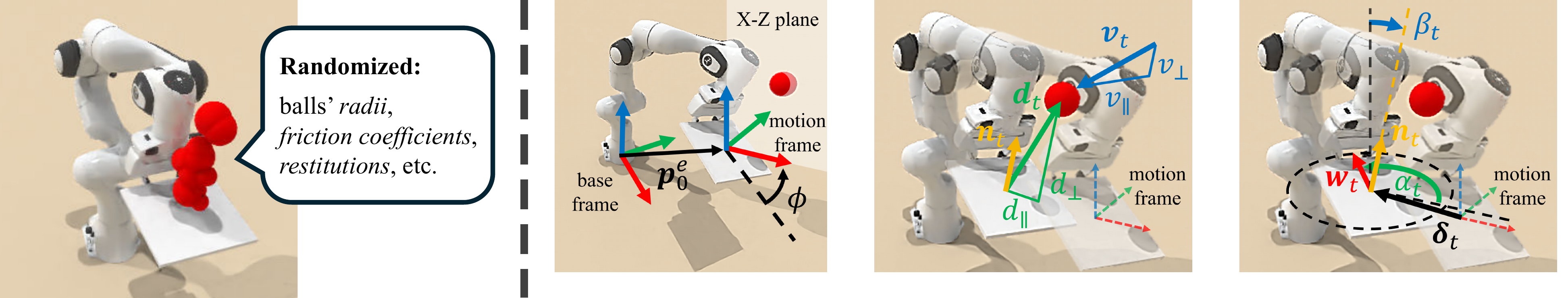}};
        \node[anchor=west, align=left] at (7.8, 0.1) {(a)};
        \node[anchor=west, align=left] at (11.8, 0.1) {(b)};
        \node[anchor=west, align=left] at (16.0, 0.1) {(c)};
    \end{tikzpicture}
    \caption{
        The reactive catching problem instantiation.
        \emph{Left:} Randomized parameters in DRIS.
        \emph{Right:}
        Components used to represent the problem:
        (a) The motion frame is defined at the plate's initial position $\bm{p}^e_0$ and rotated about the Z-axis by angle $\phi$ so that its X-Z plane passes through the incoming ball.
        (b)
        The robot has moved from the initial ghosted to the current opaque configuration.
        The system state consists of the ball's relative displacement $\bm{d}_t$ (green arrow) to the plate and its linear velocity $\bm{v}_t$ (blue arrow), both expressed in the motion frame;
        these two quantities are decomposed with respect to the plate's normal vector $\bm{n}_t$ (yellow arrow) to define the reward.
        (c) The action comprises a translation of the plate $\bm{\delta}_t$ (black arrow) and a tilting configuration: angle $\alpha_t$ (green) specifies a horizontal axis $\bm{w}_t$ (red arrow), about which the plate is rotated by angle $\beta_t$ (blue).
    }
    \label{fig:task}
\end{figure*}

\subsection{Zero-Shot Sim-to-Real Learning}
\label{sec:s2r}

Since DRIS defined in Sec.~\ref{sec:dris} inherently embeds uncertainties into its representation, 
integrating DRIS into policy learning in simulation enables more robust sim-to-real transfer without fine-tuning.
Because the DRIS size $N$ may vary and because the robot will manipulate only a single instance (although trained with multiple instances in DRIS) in real deployment, 
the policy must take an input whose dimensionality is independent of $N$.
Thereafter,
we require a compact and size-agnostic representation of DRIS.
To this end,
we employ an encoder $\psi(\cdot)$ that maps the DRIS state (i.e., $\mathcal{S}_t)$ to a fixed-dimensional latent vector $\bm{z}_t = \psi \left(\mathcal{S}_t \right) \in \mathbb{R}^{d_z}$,
where $d_z$ is the latent dimension and remains constant regardless of the size $N$ of $\mathcal{D}_t$.
A policy that directly operates in this latent space is then learned by solving:
\begin{equation}
\begin{aligned}
    \max_{\pi_{\bm{\theta}}} \quad &
    \mathbb{E}_{\pi_{\bm{\theta}}}\!\left[
        \sum_{t=0}^{T-1} \gamma^t \cdot
        \frac{1}{N}\sum_{\bm{s}^{(i)}_t \in \mathcal{S}_t}
        r(\bm{s}^{(i)}_t, \, \bm{a}_t)
    \right]\\
    \text{s.t.}
    \quad & \bm{a}_t = \pi_{\bm{\theta}}\left(\bm{z}_t \right) \in \mathcal{A}, \quad \, \,   \forall t = 0, \cdots, T - 1, \\
    \quad & \bm{z}_t = \psi\left( \mathcal{S}_t\right) \in \mathbb{R}^{d_z}, \quad \forall t = 0, \cdots, T - 1, \\
    \quad & \mathcal{S}_t = \text{proj}_{\mathcal{S}} \left( \mathcal{D}_t \right), \quad \, \, \, \, \, \forall t = 0, \cdots, T - 1, \\
    \quad & \mathcal{D}_{t+1} = \mathcal{F}\left(\mathcal{D}_t, \, \bm{a}_t \right) \subset \mathcal{S} \times \mathcal{C}
\end{aligned}
\end{equation}
Compared to conventional DR, we summarize the benefits of DRIS in Theorem~\ref{theorem:dris}:

\begin{theorem}[DRIS improves stability and robustness]
\label{theorem:dris}
    Relative to conventional domain randomization, DRIS is an exact particle approximation for the system belief propagation.
    Consequently, DRIS yields a more stable optimization, more robust policies, and improved sim-to-real generalization.
    (See \hyperref[appendix:math]{Appendix B} for detailed analysis and proofs.)
\end{theorem}

As illustrated in Fig.~\ref{fig:task} (left) with a catching task,
DRIS can typically be implemented by a set of multiple object instances,
each initialized with different physical properties (domain parameters) such as shape, weight, and friction.
The instances evolve independently, i.e., they do not interact with one another, 
and their respective state transitions under the same robot action are evaluated in parallel for policy update.

\begin{figure}[t]
\begin{algorithm}[H]
\caption{DRIS-based Policy Learning and Inference}
    \begin{algorithmic}[1]
        \Statex \Require Domain parameter space $\mathcal{C}$, initial policy $\pi_{\bm{\theta}}$, DRIS encoder $\psi$, dynamics $\mathcal{F}$ modeled by a simulator
        \Statex
        
        \Statex \textbf{\emph{Part (a): Policy Learning in Simulation}}
        \While{training}
            \State $\bm{s}_0 \gets \Call{InitState}{\null}$ \hfill \Comment{Random Initialization}
            \State $\mathcal{D}_0 \gets \{\}$  \hfill \Comment{Initialize DRIS}
            \For{$i = 1, \cdots, N$}
                \State $\bm{c}^{(i)} \sim \mathcal{U}\left(\mathcal{C}\right)$, $\bm{s}_0^{(i)} \gets \bm{s}_0$
                \State $\mathcal{D}_0 \gets \mathcal{D}_0 \cup \left\{ \left(\bm{s}_0^{(i)}, \, \bm{c}^{(i)} \right) \right\}$
            \EndFor
            \State $\mathcal{T} \gets \{\}$ \hfill \Comment{Simulation Rollout}
            \For{$t = 0, \cdots, T-1$}
                \State $\mathcal{S}_t \gets \text{proj}_{\mathcal{S}} \left( \mathcal{D}_t \right)$ \hfill \Comment{DRIS State}
                \State $\bm{z}_t \gets \psi \left( \mathcal{S}_t \right)$, $\bm{a}_t \gets \pi_{\bm{\theta}} \left(\bm{z}_t\right)$ 
                \State $r_t \gets \frac{1}{N} \sum_{\bm{s} \in \mathcal{S}_t} r\left(\bm{s}, \, \bm{a}_t\right)$ \hfill \Comment{Average Reward}
                \State $\mathcal{D}_{t+1} \gets \mathcal{F} \left(\mathcal{D}_t, \, \bm{a}_t\right)$ \hfill \Comment{Propagate by Eq.~\eqref{eq:dris_prop}}
                \State $\mathcal{T} \gets \mathcal{T} \cup \left(\bm{z}_t, \bm{a}_t, r_t\right)$
            \EndFor
            \State $\pi_{\bm{\theta}} \gets \Call{PolicyUpdate}{\pi_{\bm{\theta}}, \mathcal{T}}$ \hfill \Comment{Policy Update}
        \EndWhile
        \Statex
        
        \Statex \textbf{\emph{Part (b): Zero-Shot Inference in Real World}}
        \While{\textbf{not} finished}
            \State $\bm{s} \gets \Call{ObserveState}{\null}$ \hfill \Comment{Real-World State}
            \State $\bm{z} \gets \psi \left( \left\{ \bm{s} \right\} \right)$, $\bm{a} \gets \pi_{\bm{\theta}} \left( \bm{z} \right)$ \hfill \Comment{Generate Action}
            \State finished $\gets \Call{Execute}{\bm{a}}$
        \EndWhile
    \end{algorithmic}
    \label{alg:dris}
\end{algorithm}
\end{figure}

We summarize the DRIS-based policy learning pipeline (using an on-policy example) and its zero-shot inference in Alg.~\ref{alg:dris}.
At each training iteration in simulation, we randomly initialize the system state $\bm{s}_0$ and sample $N$ different domain parameters from $\mathcal{C}$ to construct the DRIS,
with all instances initialized at $\bm{s}_0$.
We then collect on-policy rollouts to update the policy, where the rollout stores the encoded latent state $\bm{z}_t$ rather than the full set of instance states $\left\{\bm{s}_t^{(i)}\right\}_{i=1}^N$.
After training, the policy can be deployed in the real world without fine-tuning:
At each time step,
the observed single real-world state $\bm{s}$ (equivalent to a DRIS of size $N=1$) is encoded and provided as input to the trained policy to infer the next action,
repeatedly until the task is finished.

%% file: includes/problem.tex
\section{Reactive Catching Case Study}
\label{sec:prob}

We instantiate the proposed DRIS-based learning strategy on the problem of reactive ball catching with an $M$-DoF robot manipulator.
As illustrated in Fig.~\ref{fig:first_pic}, 
the robot is equipped with a flat plate rigidly attached to its end-effector and is required to catch a ball thrown toward it.
To successfully perform this task, 
the robot must approach and make contact with the incoming ball, absorb its kinetic energy through controlled motions, and subsequently stabilize and balance the ball once it comes to rest on the plate.

\subsection{Problem Representation}
\label{sec:problem_form}

To enable a more compact problem representation and ensure the catching policy is invariant to the ball's incoming direction,
we define a fixed motion frame relative to the ball at $t=0$ (the beginning of manipulation).
All task-related quantities (e.g., state and action) are expressed in this motion frame.
As shown in Fig.~\ref{fig:task} (right-a),
the motion frame is constructed with its origin at the plate's center, its $Z$-axis aligned vertically upward, and its $X$-$Z$ plane passing through the ball's center.
The transformation from the robot's base frame to the motion frame is denoted by ${}^{b}\bm{T}_m = \left(\bm{R}_z(\phi), \, \bm{p}^e_0 \right) \in SE(3)$,
where $\bm{R}_z(\phi) \in SO(3)$ denotes a rotation about the $Z$-axis by angle $\phi$ and $\bm{p}^e_0 \in \mathbb{R}^3$ represents the plate's center position at $t=0$, expressed in the robot's base frame.

\textbf{State.}
The state of the system consists of the ball's displacement relative to the plate's center and its linear velocity, both expressed in the motion frame.
Specifically, as shown in Fig.~\ref{fig:task} (right-b),
the state is denoted as $\bm{s}_t = \left(\bm{d}_t, \, \bm{v}_t \right) \in \mathbb{R}^6$, where $\bm{d}_t = \bm{R}_{z}(\phi)^\top \left(\bm{p}^o_t - \bm{p}^e_t \right) \in \mathbb{R}^3$ is the displacement of the ball (with position $\bm{p}^o_t$) relative to the plate's center (with position $\bm{p}^e_t$) at time $t$, expressed in the motion frame;
$\bm{v}_t = \bm{R}_{z}(\phi)^\top \bm{v}^o_t \in \mathbb{R}^3$ is the ball's linear velocity $\bm{v}^o_t$ transformed into the motion frame.

\textbf{Action.}
The action is defined as $\bm{a}_t = \left(\bm{\delta}_t, \, \bm{u}_t \right) \in \mathbb{R}^5$,
where $\bm{\delta}_t \in \mathbb{R}^3$ denotes the displacement command of the plate's center expressed in the motion frame, 
i.e., the corresponding target position of the plate's center in the robot's base frame is computed as $\Tilde{\bm{p}}^e_t = \bm{R}_{z}(\phi)\, \bm{\delta}_t + \bm{p}^e_t$.
The remaining action component
$\bm{u}_t = (\alpha_t, \, \beta_t) \in \mathbb{R}^2$ represents the target tilting configuration of the plate in 3D.
Specifically, the plate is rotated about a horizontal axis $\bm{w}_t = \left( \cos{\alpha_t}, \, \sin{\alpha_t}, \, 0 \right)^\top$ by an angle $\beta_t$, where $\alpha_t \in \left[0, \, 2\pi \right)$ and $\beta_t \in \left[0, \, \pi / 4 \right)$.
This rotation tilts the plate's normal vector $\bm{n}_t \in \mathbb{R}^3$ as illustrated in Fig.~\ref{fig:task} (right-c).
During execution,
the complete action $\bm{a}_t$ is converted into a desired plate pose, which is then used to command the robot via joint torque control,
as will be detailed in Sec.~\ref{sec:control}.

\textbf{Reward.}
Given the observed state $\bm{s}_t = \left(\bm{d}_t, \, \bm{v}_t \right)$ and the plate normal $\bm{n}_t \in \mathbb{R}^3$ with $\lVert \bm{n}_t \rVert = 1$ (all in the motion frame),
we decompose the ball's displacement and velocity into components perpendicular and parallel to the plate surface $d_{\perp}$, $d_{\parallel}$, $v_{\perp}$, $v_{\parallel} \in \mathbb{R}$, as in Fig.~\ref{fig:task} (right-b):
\begin{equation}
\begin{aligned}
    d_{\perp} &= \bm{d}_t \cdot \bm{n}_t, \quad d_{\parallel} = \lVert \bm{d}_t - d_{\perp} \cdot \bm{n}_t \rVert \\
    v_{\perp} &= \bm{v}_t \cdot \bm{n}_t, \quad v_{\parallel} = \lVert \bm{v}_t - v_{\perp} \cdot \bm{n}_t \rVert
\end{aligned}
\end{equation}
The reward is the sum of a velocity term $r_v \in \left(0, \, 1\right]$ (which yields a higher value when the ball's velocity is kept lower) and a constant penalty term $r_p = -1$ applied when the ball moves beneath or outside the plate: 
\begin{equation}
\begin{aligned}
     r_t &= r(\bm{s}_t, \, \bm{a}_t) = r_v + r_p, \\
     r_v &= \frac{1}{2}\exp \left(-\frac{ v_{\parallel}^2}{\eta^2} \right) + \frac{1}{2}\exp \left(- \frac{\max\left\{v_{\perp}, \, -0.1\right\}^2}{\eta^2}\right),\\
     r_p &= -\mathbbm{1}\left\{ d_{\perp} < 0 \lor d_{\parallel} > l_e\right\}
\end{aligned}
\label{eq:reward}
\end{equation}
where $\eta \in \mathbb{R}$ in the $r_v$ term is a decay coefficient controlling the sensitivity to ball velocity, 
and $l_e \in \mathbb{R}$ in $r_p$ term denotes the plate's size (e.g., half length).

\begin{figure*}[t]
    \centering
    \includegraphics[width=\linewidth]{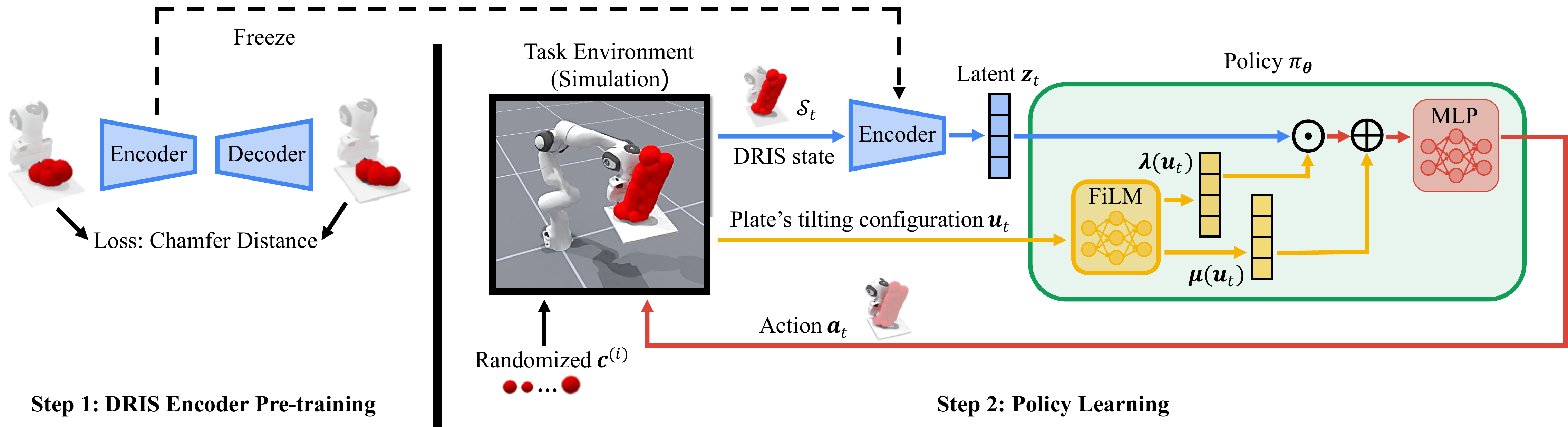}
    \caption{
        Schematic overview of the DRIS-based learning pipeline.
        \emph{Step 1 (Left):} The DRIS encoder is pre-trained to reconstruct the input DRIS state.
        \emph{Step 2 (Right):} The policy network is trained via RL using the pre-trained (and frozen) DRIS encoder.
        At the beginning of each simulation episode, 
        the physical property $\bm{c}^{(i)}$ of each ball instance in the DRIS is randomized; 
        at each time step,
        the DRIS state $\mathcal{S}_t$ is encoded and fed to the policy $\pi_{\bm{\theta}}$ to generate an action $\bm{a}_t$.
    }
    \label{fig:learning}
\end{figure*}

\subsection{Training with DRIS}
\label{sec:train_dris}

For the reactive catching problem, 
we consider four physical properties as the domain parameters: the ball's radius, static friction, dynamic friction, and restitution coefficient, i.e., $\mathcal{C} \subset \mathbb{R}^4$.
As shown in Fig.~\ref{fig:learning}, 
in each episode, we spawn $N$ balls whose physical properties are independently sampled from $\mathcal{C}$ to construct the DRIS.
In a vectorized representation, the DRIS state is formed by concatenating the states (i.e., relative displacements and velocities) of these ball instances,
i.e., $\mathcal{S}_t = \left(\bm{s}_t^{(1)}, \cdots, \bm{s}_t^{(N)} \right) \in \mathbb{R}^{N \times 6}$.

As described in Sec.~\ref{sec:s2r}, an encoder $\psi(\cdot)$ is required to transform the DRIS state into a latent vector in a compact feature space.
The encoder must be size-agnostic to the DRIS size $N$, mapping any set to a fixed-dimensional latent feature $\bm{z}_t \in \mathbb{R}^{d_z}$ where $d_z$ denotes the dimension of the latent feature space.
We adopt a point cloud-based Autoencoder (AE)~\cite{achlioptas2018learning} and extend its input dimensionality from 3D to 6D by incorporating both the positions and velocities of balls.
The point cloud AE consists of an encoder built from 1D convolutional layers followed by a feature-wise max-pooling layer to extract features of the input set~\cite{qi2017pointnet}, and an MLP-based decoder that reconstructs the original set from the latent feature.
We collect a dataset of DRIS samples by executing random robot actions, 
and pre-train this AE to reconstruct DRIS state using the Chamfer distance as reconstruction loss: 
\begin{equation}
\mathcal{L}_{\psi}(\mathcal{S}_t, \widetilde{\mathcal{S}}) = \frac{1}{N} \sum_{\bm{s} \in \mathcal{S}_t} \min_{\bm{s}' \in \widetilde{\mathcal{S}}} \lVert \bm{s} - \bm{s}' \rVert^2 + \frac{1}{\lvert \widetilde{\mathcal{S}} \rvert} \sum_{\bm{s}' \in \widetilde{\mathcal{S}}} \min_{\bm{s} \in \mathcal{S}_t} \lVert \bm{s} - \bm{s}' \rVert^2
\end{equation}
where $\mathcal{S}_t$ and $\widetilde{\mathcal{S}}$ denote the input and reconstructed DRIS state, respectively.
The encoder from the pre-trained AE is then used as $\psi(\cdot)$ for downstream policy learning.

With RL,
at each training step, 
the pre-trained encoder $\psi$ transforms the observed DRIS state $\mathcal{S}_t$ into a latent vector $\bm{z}_t$;
the policy, operating in the latent feature space (as detailed in Sec.~\ref{sec:policy}), 
is trained using a policy gradient algorithm such as Proximal Policy Optimization (PPO)~\cite{schulman2017proximal}.

\subsection{FiLM-Conditioned Policy Network}
\label{sec:policy}

The policy $\pi_{\bm{\theta}}$ consists of two components: a Feature-wise Linear Modulation (FiLM)~\cite{perez2018film} module and an MLP-based action generation network,
as illustrated in Fig.~\ref{fig:learning}.

At each step $t$, given the encoded state feature $\bm{z}_t$ and the observed plate tilting orientation $\bm{u}_t$ (which serves as the conditional signal), the FiLM module applies a feature-wise affine transformation:
\begin{equation}
    \Tilde{\bm{z}}_t = \text{FiLM}\left(\bm{z}_t, \, \bm{u}_t \right) = \bm{\lambda}\left(\bm{u}_t\right) \odot \bm{z}_t + \bm{\mu}\left(\bm{u}_t\right)
\end{equation}
where $\bm{\lambda}\left(\bm{u}_t\right)$, $\bm{\mu}\left(\bm{u}_t\right) \in \mathbb{R}^{d_z}$ are outputs of two small neural networks that generate scaling and bias coefficients, 
and $\odot$ denotes element-wise multiplication.
The FiLM-modulated feature $\Tilde{\bm{z}}_t \in \mathbb{R}^{d_z}$ is then passed through an MLP-based action-generation network to produce the action $\bm{a}_t$:
\begin{equation}
    \bm{a}_t = \pi_{\bm{\theta}} (\bm{z}_t) = \text{MLP}\left(\Tilde{\bm{z}}_t\right) = \text{MLP} \left(\text{FiLM}\left(\bm{z}_t, \, \bm{u}_t\right)\right)
\end{equation}
The policy parameters $\bm{\theta}$ include the learnable weights of the FiLM modulation networks $\bm{\lambda}(\cdot)$ and $\bm{\mu}(\cdot)$ as well as those of the action-generation MLP.
During training, all parameters are jointly optimized through gradient-based updates.

The use of FiLM conditioning is motivated by the contact physics of the catching task.
When the ball rests on the plate,
the plate’s orientation (tilting angle) directly alters the tangential component of gravity along the plate surface and the contact direction, 
thereby affecting the ball’s acceleration.
By modulating the latent feature with the plate's orientation, 
the policy can dynamically adapt its action generation to account for these gravity-induced variations,
enabling more physically consistent control across different tilt configurations.

\subsection{Control Implementation on High-DoF Manipulator}
\label{sec:control}

As described in Sec.~\ref{sec:problem_form},
the action $\bm{a}_t$ consists of a desired plate displacement $\bm{\delta}_t$ and a tilting configuration $\bm{u}_t = \left( \alpha_t, \, \beta_t \right)$ which corresponds to a rotation of the plate about a horizontal axis $\bm{w}_t = \left( \cos{\alpha_t}, \,  \sin{\alpha_t}, \, 0 \right)^\top$ by an angle $\beta_t$ (all quantities expressed in the motion frame).
The desired plate pose (expressed in the robot's base frame) is given by: $\widetilde{\bm{T}}_t = \left(\widetilde{\bm{R}}^e_t, \, \Tilde{\bm{p}}^e_t\right) \in SE(3)$,
where the orientation and position components are computed using Rodrigues' formula ($\hat{\bm{w}_t} \in \mathbb{R}^{3\times 3}$ is the skew-symmetric matrix of $\bm{w}_t$):
\begin{equation}
\begin{aligned}
    \widetilde{\bm{R}}^e_t &= \bm{R}_z (\phi) \left(\mathbb{I} + \hat{\bm{w}_t} \sin{\beta_t} + \hat{\bm{w}_t}^2 \left(1 - \cos{\beta_t}\right) \right) \in SO(3) \\
    \Tilde{\bm{p}}^e_t &= \bm{R}_{z}(\phi) \, \bm{\delta}_t + \bm{p}^e_t \in \mathbb{R}^3
\end{aligned}
\end{equation}

The desired joint configuration of the robot $\Tilde{\bm{q}}_t$ is obtained via inverse kinematics ($M$ being the robot's DoF): $\Tilde{\bm{q}}_t = \text{IK} \left( \widetilde{\bm{T}}_t \right) \in \mathbb{R}^M$.
The desired joint configuration is tracked using a joint-space torque controller that generates real-time joint torques (no inertial feedforward term is included, as the desired joint acceleration is zero):
\begin{equation}
    \bm{\tau} = \bm{K} \left(\Tilde{\bm{q}}_t - \bm{q} \right) - \bm{D} \dot{\bm{q}} + \bm{C}\left(\bm{q}, \, \dot{\bm{q}}\right) + \bm{g}(\bm{q}) \in \mathbb{R}^M
\label{eq:control}
\end{equation}
where $\bm{q}$, $\dot{\bm{q}} \in \mathbb{R}^M$ are the observed joint angles and velocities from sensor readings;
$\bm{K}$, $\bm{D} \in \mathbb{R}^{M\times M}$ are the stiffness and damping gain matrices;
$\bm{C}\left(\bm{q}, \, \dot{\bm{q}}\right)$, $\bm{g}(\bm{q}) \in \mathbb{R}^M$ are the Coriolis and gravity terms, respectively.

%% file: includes/expsim.tex
\section{Training and Validation in Simulation}
\label{sec:expsim}

In this section, we implement the reactive catching task, 
train the policies purely in simulation using our proposed strategy, 
and evaluate its performance under various sources of uncertainty.

\subsection{Training in Simulation}
\label{sec:train_sim}
The task simulation was implemented in ManiSkill~\cite{tao2025maniskill3}.
We instantiated $128$ parallel simulation environments for training, each containing a robot manipulating a set of multiple different balls (i.e., DRIS) for up to $20$ steps (corresponding to one second of simulated time).
When launching each environment instance, the balls are randomly positioned at a horizontal distance between $\left[1.0, \, 2.0\right]m$ away from the plate and assigned an initial velocity such that they reach a predefined catching region after $\left[1.0, \, 1.5\right]s$, approaching from a random incoming direction.
The catching region is defined as a zone at the same height as the plate and within a horizontal radius of $0.2m$ from its center.
Each episode begins when the ball is about to enter this region, 
specifically within a random lead time of $\left[0.08, 0.12\right]s$ before arrival,
at which point the robot starts executing generated actions.
The initial joint angles of the robot are sampled from a Gaussian distribution centered at the home configuration with a standard deviation of $0.02$ radians per joint.
We disable collision detection in simulation between balls within the same DRIS and between balls and all robot links except the plate (i.e., only ball-plate collisions are enabled) to eliminate unmodeled contacts for stable training.

We first collected training data for the DRIS encoder by allowing the robot to interact with the DRIS containing $200$ balls using random actions for $50$ episodes.
Across all $128$ parallel environment instances, this resulted in a total of $128,000$ recorded DRIS state samples (i.e., the states of all balls at each step).
The encoder was then trained on this dataset for $100$ epochs, which took approximately $10$ minutes.
We then froze the pre-trained DRIS encoder and used it to train the FiLM-conditioned policy for $1000$ epochs with PPO, which took approximately $2$ hours.

\begin{figure*}[t]
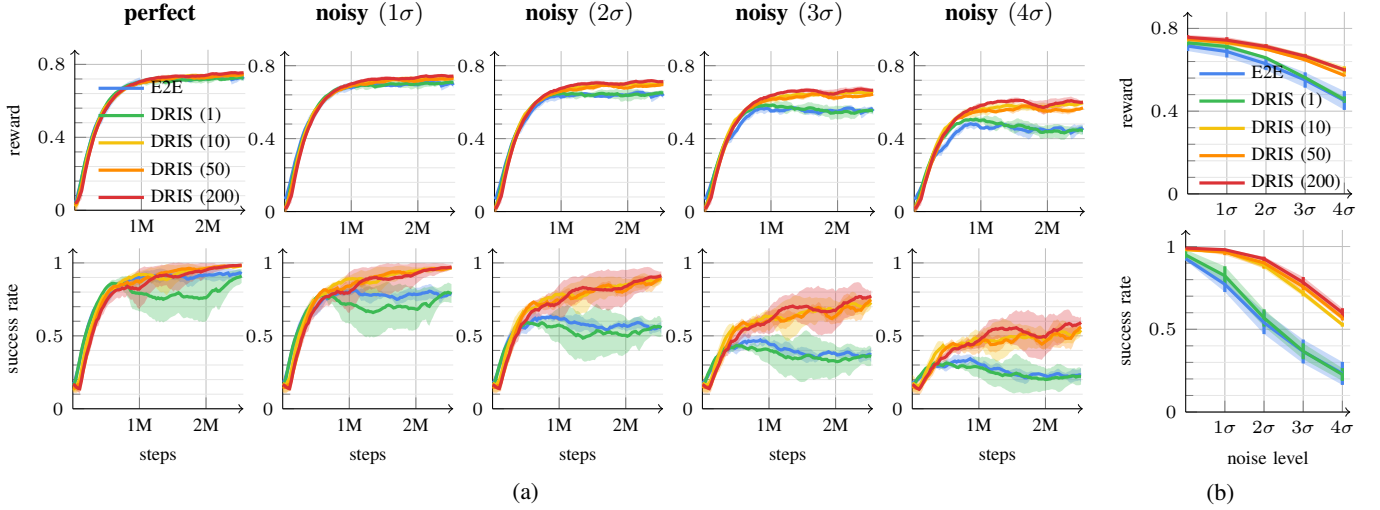

    \centering

    \begin{minipage}[t]{0.77\linewidth}
        \centering
        \begin{minipage}[b]{\linewidth}
            \begin{minipage}{0.19\linewidth}
                \input{figs/eval_reward.tex}  
            \end{minipage}
            \begin{minipage}{0.19\linewidth}
                \input{figs/eval_1_reward.tex}  
            \end{minipage}
            \begin{minipage}{0.19\linewidth}
                \input{figs/eval_2_reward.tex}  
            \end{minipage}
            \begin{minipage}{0.19\linewidth}
                \input{figs/eval_3_reward.tex}  
            \end{minipage}
            \begin{minipage}{0.19\linewidth}
                \input{figs/eval_4_reward.tex}  
            \end{minipage}
        \end{minipage}
    
        \begin{minipage}[b]{\linewidth}
            \vspace{-10pt}
            \begin{minipage}{0.19\linewidth}
                \input{figs/eval_success.tex}  
            \end{minipage}
            \begin{minipage}{0.19\linewidth}
                \input{figs/eval_1_success.tex}  
            \end{minipage}
            \begin{minipage}{0.19\linewidth}
                \input{figs/eval_2_success.tex}  
            \end{minipage}
            \begin{minipage}{0.19\linewidth}
                \input{figs/eval_3_success.tex}
            \end{minipage}
            \begin{minipage}{0.19\linewidth}
                \input{figs/eval_4_success.tex}  
            \end{minipage}
        \end{minipage}
        \vspace{-15pt}
        \subcaption{}
    \end{minipage}\hspace{0.035\linewidth} 
    \begin{minipage}[t]{0.16\linewidth}  
        \centering
        \begin{minipage}{\linewidth}
            \input{figs/noise_reward.tex}  
        \end{minipage}
        
        \begin{minipage}{\linewidth}
            \vspace{-10pt}
            \input{figs/noise_success.tex}  
        \end{minipage}
        \vspace{-15pt}
        \subcaption{}
    \end{minipage}

    \caption{
        Reward (top) and success rate (bottom) under varying observation noise.
        (a) Training curves across simulation steps, evaluated under progressively larger noise levels (left to right);
        (b) Evaluation of each policy at each noise level, averaged over the last $100$ epochs.
    }
    \label{fig:obs_eval}
\end{figure*}

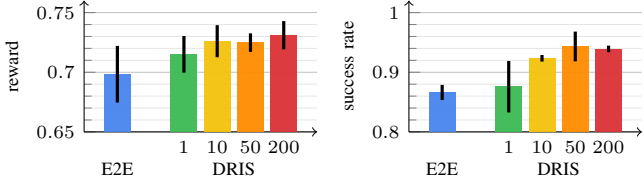
\begin{figure}[t]
    \centering
    \begin{minipage}{0.49\columnwidth}
        \input{figs/exec_reward.tex}  
    \end{minipage}
    \begin{minipage}{0.49\columnwidth}
        \input{figs/exec_success.tex}  
    \end{minipage}
    \caption{
        Reward (left) and success rate (right) of each policy under execution errors, averaged over the last $100$ epochs.
    }
    \label{fig:exec_eval}
\end{figure}

\begin{figure}[t]
    \centering
    \begin{minipage}{0.49\columnwidth}
        \input{figs/obj_reward.tex}  
    \end{minipage}
    \begin{minipage}{0.49\columnwidth}
        \input{figs/obj_success.tex}  
    \end{minipage}
    \caption{
        Reward (left) and success rate (right) of each policy when manipulating balls with unseen restitution values (higher than those in training), averaged over the last $100$ epochs.
    }
    \label{fig:exec_obj}
\end{figure}
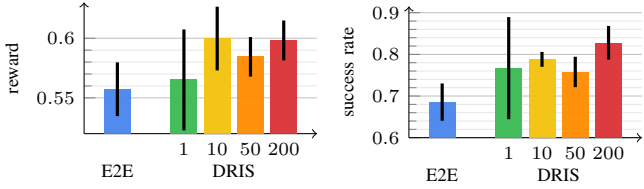

\subsection{Evaluation against Uncertainties}

To evaluate the effect of DRIS size on policy performance, we trained separate policies with different numbers of balls in the DRIS, specifically $N = 1, 10, 50, 200$.
Moreover, to compare training with the proposed DRIS-based strategy against a standard RL setting, 
we trained an additional baseline policy using the same neural network architecture but learned end-to-end directly from the observed single-ball state to the desired action.
We refer to this baseline as \emph{E2E}.
During evaluation, the DRIS-trained policies, although trained with multiple balls in each DRIS, were always tested on catching a single ball to reflect the real task setting.

We evaluate robustness against the following three sources of uncertainty:

\emph{Observation Noise.}
This setting mimics scenarios where the ball's state is not accurately estimated (e.g., due to camera calibration errors or tracking inaccuracies).
As shown in Fig.~\ref{fig:obs_eval},
we inject Gaussian noise into the observed ball states at different noise levels,
obtained by scaling a base standard deviation $\sigma$ ($1\,cm$ for position and $5\,cm/s$ for velocity) by factors ranging from $1$ to $4$.
We report both the final episode reward and success rate,
where success is defined as the ball remaining on the plate with a velocity below $0.1\, m/s$ at the end of the episode.
The E2E baseline and the DRIS policy with $N=1$, both trained using a single ball (where the DRIS policy reduces to a standard RL setting but with a pre-trained encoder), 
perform comparably under perfect observations.
However, as the observation noise increases, 
their performance degrades significantly.
Moreover, they are more prone to overfitting to the perfect observations, as can be seen from their gradual performance degradation when trained for a longer period in Fig.~\ref{fig:obs_eval}.
In contrast, DRIS policies trained with even a slightly larger number of balls (e.g., $10$) exhibit substantially improved robustness, demonstrating greater tolerance to observation noise.

\emph{Execution Error.}
Then, we added noise to the desired joint positions by uniformly sampling perturbations from $\left[-0.05, \, 0.05\right]\, rad$ to simulate robot execution errors (e.g., controller's tracking inaccuracies).
We report the final reward and success rate for each policy in Fig.~\ref{fig:exec_eval}.
Similar to the observation noise evaluation,
the E2E baseline is notably less robust than policies trained with DRIS.
While increasing the DRIS size generally improved robustness in this evaluation,
even a modest number of balls (e.g., $10$) provided an obvious boost in performance under execution noise.

\emph{Out-of-Distribution Physics.}
To evaluate how the policies perform under unseen physical parameters,
we tested them on balls with higher and more challenging restitution coefficients in $\left[0.7, 0.8\right]$,
despite being trained only within $\left[0.4, \, 0.7\right]$.
The reward and success rate are reported in Fig.~\ref{fig:exec_obj}.
Compared with both the E2E baseline and the DRIS policy trained with a single ball,
policies trained with larger DRIS generalize better to manipulate balls with unseen physical properties (restitution, in this case).

\subsection{Scalability and Computational Efficiency}

To evaluate the efficacy of DRIS from a computational fairness perspective, 
we conducted a scaling study to verify that the observed performance gains are not merely a product of increased ball interactions.
Specifically, we expanded the E2E baseline to $128 \times 50$ single-ball environments and compared it against our DRIS configuration, which utilizes $128$ environments with a DRIS size of $50$ (i.e., $50$ balls per environment), thereby matching the total number of ball interactions across both methods.

Evaluated under $2\sigma$ observation noise, the results indicate that while massive parallelization improves the E2E baseline success rate from $0.55$ to $0.73$, it does so at a substantial computational cost, increasing simulation VRAM usage from $0.65$~GB to $2$~GB. 
In contrast, the DRIS configuration achieves a significantly higher success rate of $0.89$ while increasing VRAM only modestly to $0.71$~GB. 
These results demonstrate that DRIS-based propagation provides superior robustness while remaining considerably more resource-efficient than standard environment parallelization.

%% file: figs/noise_reward.tex
\begin{tikzpicture}

\definecolor{mycolor2}{rgb}{0.23529,0.72941,0.32941}%
\definecolor{mycolor3}{rgb}{0.95686,0.76078,0.05098}%
\definecolor{mycolor5}{rgb}{0.85882,0.19608,0.21176}%
\definecolor{mycolor1}{rgb}{0.28235,0.52157,0.92941}%
\definecolor{mycolor4}{rgb}{1.00000,0.54902,0.00000}%

\begin{axis}[
width=0.75\linewidth,
height=0.83\linewidth,
scale only axis,
scaled ticks=false, 
grid=both,
grid style={line width=.1pt, draw=gray!20},
major grid style={line width=.2pt,draw=gray!50},
xmin=0, xmax=4.2,
xtick={1, 2, 3, 4},
xticklabels={$1\sigma$, $2\sigma$, $3\sigma$, $4\sigma$},
ylabel={reward},
ytick={0, 0.4, 0.8},
minor y tick num=4,
ymin=0, ymax=0.88,
axis x line*=bottom,
axis y line*=left,
legend style={legend cell align=left, align=left, legend columns = 1, fill=none, draw=none, at={(0.01, 0.36)},anchor=west}
]
\path [fill=mycolor1, fill opacity=0.3]
(axis cs:0,0.738572233548247)
--(axis cs:0,0.690648122995612)
--(axis cs:1,0.659511714512206)
--(axis cs:2,0.598543911749035)
--(axis cs:3,0.513616543563313)
--(axis cs:4,0.406017397939755)
--(axis cs:4,0.49702233264884)
--(axis cs:4,0.49702233264884)
--(axis cs:3,0.58743403970986)
--(axis cs:2,0.659847447202693)
--(axis cs:1,0.717009741729401)
--(axis cs:0,0.738572233548247)
--cycle;

\path [fill=mycolor2, fill opacity=0.3]
(axis cs:0,0.74828588485762)
--(axis cs:0,0.71262411355928)
--(axis cs:1,0.699832214314937)
--(axis cs:2,0.653653955070371)
--(axis cs:3,0.55597792178639)
--(axis cs:4,0.440304479713535)
--(axis cs:4,0.471450009231473)
--(axis cs:4,0.471450009231473)
--(axis cs:3,0.56756617913397)
--(axis cs:2,0.661982965858584)
--(axis cs:1,0.725283895533085)
--(axis cs:0,0.74828588485762)
--cycle;

\path [fill=mycolor3, fill opacity=0.3]
(axis cs:0,0.758007021352744)
--(axis cs:0,0.735113231733346)
--(axis cs:1,0.718412138823831)
--(axis cs:2,0.690522778649616)
--(axis cs:3,0.644977375568358)
--(axis cs:4,0.580908654472954)
--(axis cs:4,0.616445256926887)
--(axis cs:4,0.616445256926887)
--(axis cs:3,0.677167089083385)
--(axis cs:2,0.723804159343752)
--(axis cs:1,0.750813240007397)
--(axis cs:0,0.758007021352744)
--cycle;

\path [fill=mycolor4, fill opacity=0.3]
(axis cs:0,0.752638499776979)
--(axis cs:0,0.742764054893991)
--(axis cs:1,0.722867454160907)
--(axis cs:2,0.691961795600426)
--(axis cs:3,0.641211274993674)
--(axis cs:4,0.567896330476644)
--(axis cs:4,0.577948991337575)
--(axis cs:4,0.577948991337575)
--(axis cs:3,0.654990836250528)
--(axis cs:2,0.709321345376798)
--(axis cs:1,0.741587621897799)
--(axis cs:0,0.752638499776979)
--cycle;

\path [fill=mycolor5, fill opacity=0.3]
(axis cs:0,0.767729522826405)
--(axis cs:0,0.745779857991008)
--(axis cs:1,0.72891883741018)
--(axis cs:2,0.700981187232441)
--(axis cs:3,0.655833810820466)
--(axis cs:4,0.591595508471019)
--(axis cs:4,0.608048230593515)
--(axis cs:4,0.608048230593515)
--(axis cs:3,0.675090298797403)
--(axis cs:2,0.723925980520143)
--(axis cs:1,0.757466353076727)
--(axis cs:0,0.767729522826405)
--cycle;

\path [draw=mycolor1, very thick]
(axis cs:0,0.690648122995612)
--(axis cs:0,0.738572233548247);

\path [draw=mycolor1, very thick]
(axis cs:1,0.659511714512206)
--(axis cs:1,0.717009741729401);

\path [draw=mycolor1, very thick]
(axis cs:2,0.598543911749035)
--(axis cs:2,0.659847447202693);

\path [draw=mycolor1, very thick]
(axis cs:3,0.513616543563313)
--(axis cs:3,0.58743403970986);

\path [draw=mycolor1, very thick]
(axis cs:4,0.406017397939755)
--(axis cs:4,0.49702233264884);

\path [draw=mycolor2, very thick]
(axis cs:0,0.71262411355928)
--(axis cs:0,0.74828588485762);

\path [draw=mycolor2, very thick]
(axis cs:1,0.699832214314937)
--(axis cs:1,0.725283895533085);

\path [draw=mycolor2, very thick]
(axis cs:2,0.653653955070371)
--(axis cs:2,0.661982965858584);

\path [draw=mycolor2, very thick]
(axis cs:3,0.55597792178639)
--(axis cs:3,0.56756617913397);

\path [draw=mycolor2, very thick]
(axis cs:4,0.440304479713535)
--(axis cs:4,0.471450009231473);

\path [draw=mycolor3, very thick]
(axis cs:0,0.735113231733346)
--(axis cs:0,0.758007021352744);

\path [draw=mycolor3, very thick]
(axis cs:1,0.718412138823831)
--(axis cs:1,0.750813240007397);

\path [draw=mycolor3, very thick]
(axis cs:2,0.690522778649616)
--(axis cs:2,0.723804159343752);

\path [draw=mycolor3, very thick]
(axis cs:3,0.644977375568358)
--(axis cs:3,0.677167089083385);

\path [draw=mycolor3, very thick]
(axis cs:4,0.580908654472954)
--(axis cs:4,0.616445256926887);

\path [draw=mycolor4, very thick]
(axis cs:0,0.742764054893991)
--(axis cs:0,0.752638499776979);

\path [draw=mycolor4, very thick]
(axis cs:1,0.722867454160907)
--(axis cs:1,0.741587621897799);

\path [draw=mycolor4, very thick]
(axis cs:2,0.691961795600426)
--(axis cs:2,0.709321345376798);

\path [draw=mycolor4, very thick]
(axis cs:3,0.641211274993674)
--(axis cs:3,0.654990836250528);

\path [draw=mycolor4, very thick]
(axis cs:4,0.567896330476644)
--(axis cs:4,0.577948991337575);

\path [draw=mycolor5, very thick]
(axis cs:0,0.745779857991008)
--(axis cs:0,0.767729522826405);

\path [draw=mycolor5, very thick]
(axis cs:1,0.72891883741018)
--(axis cs:1,0.757466353076727);

\path [draw=mycolor5, very thick]
(axis cs:2,0.700981187232441)
--(axis cs:2,0.723925980520143);

\path [draw=mycolor5, very thick]
(axis cs:3,0.655833810820466)
--(axis cs:3,0.675090298797403);

\path [draw=mycolor5, very thick]
(axis cs:4,0.591595508471019)
--(axis cs:4,0.608048230593515);

\addplot [very thick, mycolor1]
table {%
0 0.71461017827193
1 0.688260728120804
2 0.629195679475864
3 0.550525291636586
4 0.451519865294298
};
\addlegendentry{E2E}
\addplot [very thick, mycolor2]
table {%
0 0.73045499920845
1 0.712558054924011
2 0.657818460464478
3 0.56177205046018
4 0.455877244472504
};
\addlegendentry{DRIS (1)}
\addplot [very thick, mycolor3]
table {%
0 0.746560126543045
1 0.734612689415614
2 0.707163468996684
3 0.661072232325872
4 0.598676955699921
};
\addlegendentry{DRIS (10)}
\addplot [very thick, mycolor4]
table {%
0 0.747701277335485
1 0.732227538029353
2 0.700641570488612
3 0.648101055622101
4 0.57292266090711
};
\addlegendentry{DRIS (50)}
\addplot [very thick, mycolor5]
table {%
0 0.756754690408707
1 0.743192595243454
2 0.712453583876292
3 0.665462054808935
4 0.599821869532267
};
\addlegendentry{DRIS (200)}
\end{axis}

\end{tikzpicture}

%% file: figs/noise_success.tex
\begin{tikzpicture}

\definecolor{mycolor2}{rgb}{0.23529,0.72941,0.32941}%
\definecolor{mycolor3}{rgb}{0.95686,0.76078,0.05098}%
\definecolor{mycolor5}{rgb}{0.85882,0.19608,0.21176}%
\definecolor{mycolor1}{rgb}{0.28235,0.52157,0.92941}%
\definecolor{mycolor4}{rgb}{1.00000,0.54902,0.00000}%

\begin{axis}[
width=0.75\linewidth,
height=0.83\linewidth,
scale only axis,
scaled ticks=false, 
grid=both,
grid style={line width=.1pt, draw=gray!20},
major grid style={line width=.2pt,draw=gray!50},
xmin=0, xmax=4.2,
xtick={1, 2, 3, 4},
xticklabels={$1\sigma$, $2\sigma$, $3\sigma$, $4\sigma$},
xlabel={noise level},
ylabel={success rate},
ytick={0, 0.5, 1.0},
minor y tick num=4,
ymin=0, ymax=1.1,
axis x line*=bottom,
axis y line*=left,
legend style={legend cell align=left, align=left, legend columns = 1, fill=none, draw=none, at={(0.05, 0.35)},anchor=west}
]
\path [fill=mycolor1, fill opacity=0.3]
(axis cs:0,0.938571221354225)
--(axis cs:0,0.915595445312441)
--(axis cs:1,0.725361638064512)
--(axis cs:2,0.471093258649548)
--(axis cs:3,0.293588835579794)
--(axis cs:4,0.163846475625582)
--(axis cs:4,0.302820191041084)
--(axis cs:4,0.302820191041084)
--(axis cs:3,0.437661164420206)
--(axis cs:2,0.609115074683785)
--(axis cs:1,0.818388361935488)
--(axis cs:0,0.938571221354225)
--cycle;

\path [fill=mycolor2, fill opacity=0.3]
(axis cs:0,0.980236532415877)
--(axis cs:0,0.91768013425079)
--(axis cs:1,0.766173065184037)
--(axis cs:2,0.516352891915351)
--(axis cs:3,0.31933314740277)
--(axis cs:4,0.188314696381721)
--(axis cs:4,0.265331136951612)
--(axis cs:4,0.265331136951612)
--(axis cs:3,0.41504185259723)
--(axis cs:2,0.620626274751315)
--(axis cs:1,0.880701934815963)
--(axis cs:0,0.980236532415877)
--cycle;

\path [fill=mycolor3, fill opacity=0.3]
(axis cs:0,0.986553813252889)
--(axis cs:0,0.974383686747111)
--(axis cs:1,0.96188368674711)
--(axis cs:2,0.864842218190257)
--(axis cs:3,0.698857741356693)
--(axis cs:4,0.515605622608396)
--(axis cs:4,0.540123544058271)
--(axis cs:4,0.540123544058271)
--(axis cs:3,0.73603809197664)
--(axis cs:2,0.909116115143076)
--(axis cs:1,0.974053813252889)
--(axis cs:0,0.986553813252889)
--cycle;

\path [fill=mycolor4, fill opacity=0.3]
(axis cs:0,0.98700079322229)
--(axis cs:0,0.980707540111043)
--(axis cs:1,0.948009911097417)
--(axis cs:2,0.871420487089987)
--(axis cs:3,0.736980857755291)
--(axis cs:4,0.553053241467131)
--(axis cs:4,0.588092591866202)
--(axis cs:4,0.588092591866202)
--(axis cs:3,0.777081642244708)
--(axis cs:2,0.930662846243346)
--(axis cs:1,0.98428175556925)
--(axis cs:0,0.98700079322229)
--cycle;

\path [fill=mycolor5, fill opacity=0.3]
(axis cs:0,0.99196989947994)
--(axis cs:0,0.988759267186727)
--(axis cs:1,0.974941116986747)
--(axis cs:2,0.912975738694641)
--(axis cs:3,0.751325193081151)
--(axis cs:4,0.572704655941138)
--(axis cs:4,0.627295344058862)
--(axis cs:4,0.627295344058862)
--(axis cs:3,0.815862306918849)
--(axis cs:2,0.938586761305359)
--(axis cs:1,0.986517216346587)
--(axis cs:0,0.99196989947994)
--cycle;

\path [draw=mycolor1, very thick]
(axis cs:0,0.915595445312441)
--(axis cs:0,0.938571221354225);

\path [draw=mycolor1, very thick]
(axis cs:1,0.725361638064512)
--(axis cs:1,0.818388361935488);

\path [draw=mycolor1, very thick]
(axis cs:2,0.471093258649548)
--(axis cs:2,0.609115074683785);

\path [draw=mycolor1, very thick]
(axis cs:3,0.293588835579794)
--(axis cs:3,0.437661164420206);

\path [draw=mycolor1, very thick]
(axis cs:4,0.163846475625582)
--(axis cs:4,0.302820191041084);

\path [draw=mycolor2, very thick]
(axis cs:0,0.91768013425079)
--(axis cs:0,0.980236532415877);

\path [draw=mycolor2, very thick]
(axis cs:1,0.766173065184037)
--(axis cs:1,0.880701934815963);

\path [draw=mycolor2, very thick]
(axis cs:2,0.516352891915351)
--(axis cs:2,0.620626274751315);

\path [draw=mycolor2, very thick]
(axis cs:3,0.31933314740277)
--(axis cs:3,0.41504185259723);

\path [draw=mycolor2, very thick]
(axis cs:4,0.188314696381721)
--(axis cs:4,0.265331136951612);

\path [draw=mycolor3, very thick]
(axis cs:0,0.974383686747111)
--(axis cs:0,0.986553813252889);

\path [draw=mycolor3, very thick]
(axis cs:1,0.96188368674711)
--(axis cs:1,0.974053813252889);

\path [draw=mycolor3, very thick]
(axis cs:2,0.864842218190257)
--(axis cs:2,0.909116115143076);

\path [draw=mycolor3, very thick]
(axis cs:3,0.698857741356693)
--(axis cs:3,0.73603809197664);

\path [draw=mycolor3, very thick]
(axis cs:4,0.515605622608396)
--(axis cs:4,0.540123544058271);

\path [draw=mycolor4, very thick]
(axis cs:0,0.980707540111043)
--(axis cs:0,0.98700079322229);

\path [draw=mycolor4, very thick]
(axis cs:1,0.948009911097417)
--(axis cs:1,0.98428175556925);

\path [draw=mycolor4, very thick]
(axis cs:2,0.871420487089987)
--(axis cs:2,0.930662846243346);

\path [draw=mycolor4, very thick]
(axis cs:3,0.736980857755291)
--(axis cs:3,0.777081642244708);

\path [draw=mycolor4, very thick]
(axis cs:4,0.553053241467131)
--(axis cs:4,0.588092591866202);

\path [draw=mycolor5, very thick]
(axis cs:0,0.988759267186727)
--(axis cs:0,0.99196989947994);

\path [draw=mycolor5, very thick]
(axis cs:1,0.974941116986747)
--(axis cs:1,0.986517216346587);

\path [draw=mycolor5, very thick]
(axis cs:2,0.912975738694641)
--(axis cs:2,0.938586761305359);

\path [draw=mycolor5, very thick]
(axis cs:3,0.751325193081151)
--(axis cs:3,0.815862306918849);

\path [draw=mycolor5, very thick]
(axis cs:4,0.572704655941138)
--(axis cs:4,0.627295344058862);

\addplot [very thick, mycolor1]
table {%
0 0.927083333333333
1 0.771875
2 0.540104166666667
3 0.365625
4 0.233333333333333
};
\addplot [very thick, mycolor2]
table {%
0 0.948958333333333
1 0.8234375
2 0.568489583333333
3 0.3671875
4 0.226822916666667
};
\addplot [very thick, mycolor3]
table {%
0 0.98046875
1 0.96796875
2 0.886979166666667
3 0.717447916666667
4 0.527864583333333
};
\addplot [very thick, mycolor4]
table {%
0 0.983854166666667
1 0.966145833333333
2 0.901041666666667
3 0.75703125
4 0.570572916666667
};
\addplot [very thick, mycolor5]
table {%
0 0.990364583333333
1 0.980729166666667
2 0.92578125
3 0.78359375
4 0.6
};
\end{axis}

\end{tikzpicture}

%% file: figs/exec_reward.tex
\begin{tikzpicture}

\definecolor{mycolor2}{rgb}{0.23529,0.72941,0.32941}%
\definecolor{mycolor3}{rgb}{0.95686,0.76078,0.05098}%
\definecolor{mycolor5}{rgb}{0.85882,0.19608,0.21176}%
\definecolor{mycolor1}{rgb}{0.28235,0.52157,0.92941}%
\definecolor{mycolor4}{rgb}{1.00000,0.54902,0.00000}%

\begin{axis}[
width=0.71\linewidth,
height=0.4\linewidth,
scale only axis,
scaled ticks=false, 
grid=both,
grid style={line width=.1pt, draw=gray!20},
major grid style={line width=.2pt,draw=gray!50},
xmin=0, xmax=7,
xtick={1, 3, 4, 5, 6, 4.5},
xticklabel style={align=center}, 
xticklabels={
    {{\\[1.5ex]E2E}},      
    {{$1$\\[1.5ex]}},        
    {{$10$\\[1.5ex]}},
    {{$50$\\[1.5ex]}},
    {{$200$\\[1.5ex]}},
    {{\\[1.5ex]DRIS}}
},
ylabel={reward},
ytick={0.65, 0.70, 0.75},
xmajorgrids=false,
xminorgrids=false,
ymajorgrids=true,
yminorgrids=true,
minor y tick num=4,
ymin=0.65, ymax=0.76,
axis x line*=bottom,
axis y line*=left,
xtick style={draw=none},
legend style={legend cell align=left, align=left, legend columns = 2, fill=none, draw=none, at={(0.05, 0.75)},anchor=west}
]
\draw[draw=none,fill=mycolor1, fill opacity=0.95] (axis cs:0.6,0) rectangle (axis cs:1.4,0.698356055219968);

\draw[draw=none,fill=mycolor2, fill opacity=0.95] (axis cs:2.6,0) rectangle (axis cs:3.4,0.715030952294668);

\draw[draw=none,fill=mycolor3, fill opacity=0.95] (axis cs:3.6,0) rectangle (axis cs:4.4,0.726025627056758);

\draw[draw=none,fill=mycolor4, fill opacity=0.95] (axis cs:4.6,0) rectangle (axis cs:5.4,0.724824907382329);

\draw[draw=none,fill=mycolor5, fill opacity=0.95] (axis cs:5.6,0) rectangle (axis cs:6.4,0.73106258114179);

\path [draw=black, very thick]
(axis cs:1,0.674625889786298)
--(axis cs:1,0.722086220653639);

\path [draw=black, very thick]
(axis cs:3,0.6996779392703)
--(axis cs:3,0.730383965319035);

\path [draw=black, very thick]
(axis cs:4,0.712581199098367)
--(axis cs:4,0.739470055015148);

\path [draw=black, very thick]
(axis cs:5,0.717056608823571)
--(axis cs:5,0.732593205941087);

\path [draw=black, very thick]
(axis cs:6,0.719225307451042)
--(axis cs:6,0.742899854832538);

\end{axis}

\end{tikzpicture}

%% file: figs/exec_success.tex
\begin{tikzpicture}

\definecolor{mycolor2}{rgb}{0.23529,0.72941,0.32941}%
\definecolor{mycolor3}{rgb}{0.95686,0.76078,0.05098}%
\definecolor{mycolor5}{rgb}{0.85882,0.19608,0.21176}%
\definecolor{mycolor1}{rgb}{0.28235,0.52157,0.92941}%
\definecolor{mycolor4}{rgb}{1.00000,0.54902,0.00000}%

\begin{axis}[
width=0.71\linewidth,
height=0.4\linewidth,
scale only axis,
scaled ticks=false, 
grid=both,
grid style={line width=.1pt, draw=gray!20},
major grid style={line width=.2pt,draw=gray!50},
xmin=0, xmax=7,
xtick={1, 3, 4, 5, 6, 4.5},
xticklabel style={align=center}, 
xticklabels={
    {{\\[1.5ex]E2E}},      
    {{$1$\\[1.5ex]}},        
    {{$10$\\[1.5ex]}},
    {{$50$\\[1.5ex]}},
    {{$200$\\[1.5ex]}},
    {{\\[1.5ex]DRIS}}
},
ylabel={success rate},
ytick={0.8, 0.9, 1.0},
xmajorgrids=false,
xminorgrids=false,
ymajorgrids=true,
yminorgrids=true,
minor y tick num=4,
ymin=0.8, ymax=1.02,
axis x line*=bottom,
axis y line*=left,
xtick style={draw=none},
legend style={legend cell align=left, align=left, legend columns = 2, fill=none, draw=none, at={(0.05, 0.75)},anchor=west}
]
\draw[draw=none,fill=mycolor1, fill opacity=0.95] (axis cs:0.6,0) rectangle (axis cs:1.4,0.866145833333333);

\draw[draw=none,fill=mycolor2, fill opacity=0.95] (axis cs:2.6,0) rectangle (axis cs:3.4,0.87578125);

\draw[draw=none,fill=mycolor3, fill opacity=0.95] (axis cs:3.6,0) rectangle (axis cs:4.4,0.9234375);

\draw[draw=none,fill=mycolor4, fill opacity=0.95] (axis cs:4.6,0) rectangle (axis cs:5.4,0.943229166666667);

\draw[draw=none,fill=mycolor5, fill opacity=0.95] (axis cs:5.6,0) rectangle (axis cs:6.4,0.9390625);

\path [draw=black, very thick]
(axis cs:1,0.853478764055192)
--(axis cs:1,0.878812902611474);

\path [draw=black, very thick]
(axis cs:3,0.832574051766758)
--(axis cs:3,0.918988448233242);

\path [draw=black, very thick]
(axis cs:4,0.91791322827198)
--(axis cs:4,0.92896177172802);

\path [draw=black, very thick]
(axis cs:5,0.91830795836402)
--(axis cs:5,0.968150374969313);

\path [draw=black, very thick]
(axis cs:6,0.93353822827198)
--(axis cs:6,0.94458677172802);
\end{axis}

\end{tikzpicture}

%% file: figs/obj_reward.tex
\begin{tikzpicture}

\definecolor{mycolor2}{rgb}{0.23529,0.72941,0.32941}%
\definecolor{mycolor3}{rgb}{0.95686,0.76078,0.05098}%
\definecolor{mycolor5}{rgb}{0.85882,0.19608,0.21176}%
\definecolor{mycolor1}{rgb}{0.28235,0.52157,0.92941}%
\definecolor{mycolor4}{rgb}{1.00000,0.54902,0.00000}%

\begin{axis}[
width=0.71\linewidth,
height=0.4\linewidth,
scale only axis,
scaled ticks=false, 
grid=both,
grid style={line width=.1pt, draw=gray!20},
major grid style={line width=.2pt,draw=gray!50},
xmin=0, xmax=7,
xtick={1, 3, 4, 5, 6, 4.5},
xticklabel style={align=center}, 
xticklabels={
    {{\\[1.5ex]E2E}},      
    {{$1$\\[1.5ex]}},        
    {{$10$\\[1.5ex]}},
    {{$50$\\[1.5ex]}},
    {{$200$\\[1.5ex]}},
    {{\\[1.5ex]DRIS}}
},
ylabel={reward},
ytick={0.55, 0.60},
xmajorgrids=false,
xminorgrids=false,
ymajorgrids=true,
yminorgrids=true,
minor y tick num=4,
ymin=0.52, ymax=0.63,
axis x line*=bottom,
axis y line*=left,
xtick style={draw=none},
legend style={legend cell align=left, align=left, legend columns = 2, fill=none, draw=none, at={(0.05, 0.75)},anchor=west}
]
\draw[draw=none,fill=mycolor1, fill opacity=0.95] (axis cs:0.6,0) rectangle (axis cs:1.4,0.55723130106926);

\draw[draw=none,fill=mycolor2, fill opacity=0.95] (axis cs:2.6,0) rectangle (axis cs:3.4,0.565037780006727);

\draw[draw=none,fill=mycolor3, fill opacity=0.95] (axis cs:3.6,0) rectangle (axis cs:4.4,0.599714332818985);

\draw[draw=none,fill=mycolor4, fill opacity=0.95] (axis cs:4.6,0) rectangle (axis cs:5.4,0.584420905510584);

\draw[draw=none,fill=mycolor5, fill opacity=0.95] (axis cs:5.6,0) rectangle (axis cs:6.4,0.598148853580157);

\path [draw=black, very thick]
(axis cs:1,0.534699335080428)
--(axis cs:1,0.579763267058091);

\path [draw=black, very thick]
(axis cs:3,0.522717182255958)
--(axis cs:3,0.607358377757495);

\path [draw=black, very thick]
(axis cs:4,0.573013248523797)
--(axis cs:4,0.626415417114173);

\path [draw=black, very thick]
(axis cs:5,0.567873322057368)
--(axis cs:5,0.600968488963801);

\path [draw=black, very thick]
(axis cs:6,0.581433582477318)
--(axis cs:6,0.614864124682996);

\end{axis}

\end{tikzpicture}

%% file: figs/obj_success.tex
\begin{tikzpicture}

\definecolor{mycolor2}{rgb}{0.23529,0.72941,0.32941}%
\definecolor{mycolor3}{rgb}{0.95686,0.76078,0.05098}%
\definecolor{mycolor5}{rgb}{0.85882,0.19608,0.21176}%
\definecolor{mycolor1}{rgb}{0.28235,0.52157,0.92941}%
\definecolor{mycolor4}{rgb}{1.00000,0.54902,0.00000}%

\begin{axis}[
width=0.71\linewidth,
height=0.4\linewidth,
scale only axis,
scaled ticks=false, 
grid=both,
grid style={line width=.1pt, draw=gray!20},
major grid style={line width=.2pt,draw=gray!50},
xmin=0, xmax=7,
xtick={1, 3, 4, 5, 6, 4.5},
xticklabel style={align=center}, 
xticklabels={
    {{\\[1.5ex]E2E}},      
    {{$1$\\[1.5ex]}},        
    {{$10$\\[1.5ex]}},
    {{$50$\\[1.5ex]}},
    {{$200$\\[1.5ex]}},
    {{\\[1.5ex]DRIS}}
},
ylabel={success rate},
ytick={0.6, 0.7, 0.8, 0.9},
xmajorgrids=false,
xminorgrids=false,
ymajorgrids=true,
yminorgrids=true,
minor y tick num=4,
ymin=0.6, ymax=0.915,
axis x line*=bottom,
axis y line*=left,
xtick style={draw=none},
legend style={legend cell align=left, align=left, legend columns = 2, fill=none, draw=none, at={(0.05, 0.75)},anchor=west}
]
\draw[draw=none,fill=mycolor1, fill opacity=0.95] (axis cs:0.6,0) rectangle (axis cs:1.4,0.685416666666667);

\draw[draw=none,fill=mycolor2, fill opacity=0.95] (axis cs:2.6,0) rectangle (axis cs:3.4,0.766927083333333);

\draw[draw=none,fill=mycolor3, fill opacity=0.95] (axis cs:3.6,0) rectangle (axis cs:4.4,0.788020833333333);

\draw[draw=none,fill=mycolor4, fill opacity=0.95] (axis cs:4.6,0) rectangle (axis cs:5.4,0.7578125);

\draw[draw=none,fill=mycolor5, fill opacity=0.95] (axis cs:5.6,0) rectangle (axis cs:6.4,0.827864583333333);

\path [draw=black, very thick]
(axis cs:1,0.640690151549637)
--(axis cs:1,0.730143181783696);

\path [draw=black, very thick]
(axis cs:3,0.644352413453376)
--(axis cs:3,0.889501753213291);

\path [draw=black, very thick]
(axis cs:4,0.770144789269031)
--(axis cs:4,0.805896877397636);

\path [draw=black, very thick]
(axis cs:5,0.721346726949554)
--(axis cs:5,0.794278273050446);

\path [draw=black, very thick]
(axis cs:6,0.787473966728456)
--(axis cs:6,0.868255199938211);

\end{axis}

\end{tikzpicture}

%% file: includes/expreal.tex
\section{Real-Robot Deployment}
\label{sec:expreal}

To more realistically challenge our learned policies under real-world uncertainties and evaluate their sim-to-real transfer performance,
we directly deployed the policies trained purely in simulation in a zero-shot manner on a real 7-DoF Franka Research 3 (FR3) robot manipulator,
without any fine-tuning using real-world data.

\begin{figure}[t]
    \centering
    \includegraphics[width=\linewidth]{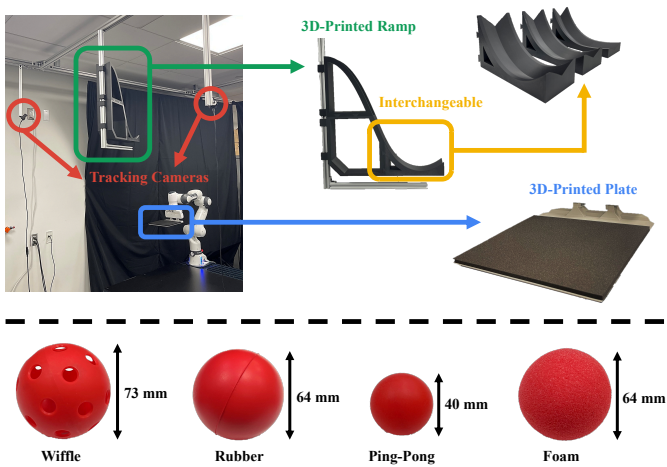}
    \caption{
        The system setup (top) and the four different balls (bottom) used for real-world experiments.
    }
    \label{fig:setup}
\end{figure}

\begin{figure}[t]
    \centering
    \begin{minipage}[b]{\linewidth}
        \includegraphics[width=\linewidth]{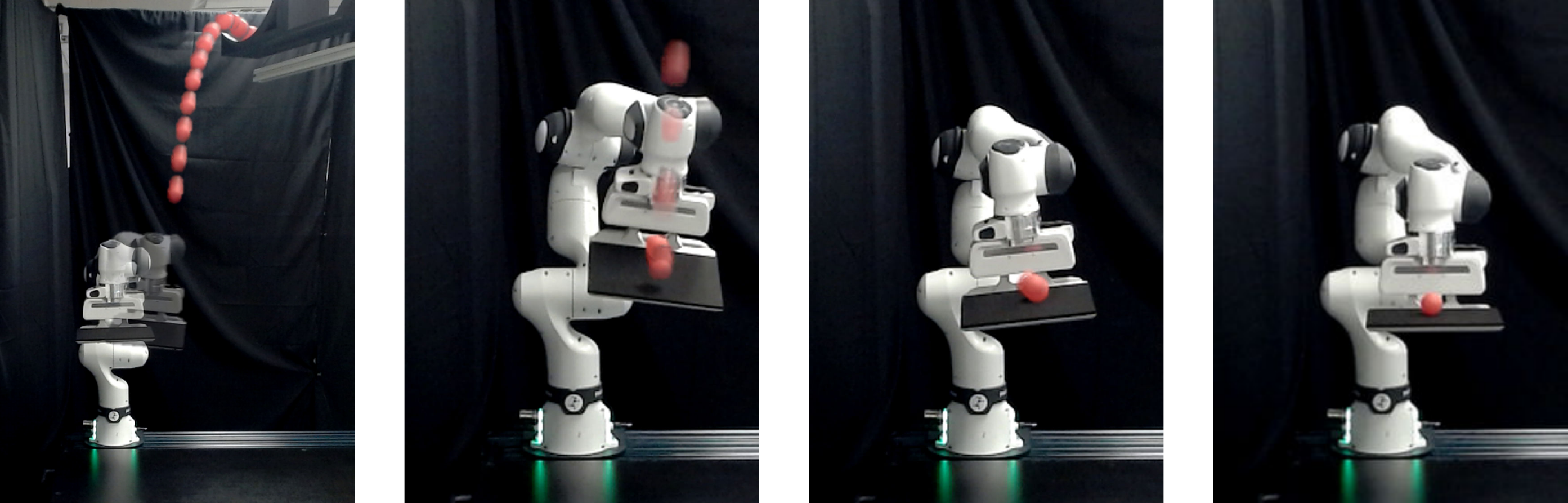}
        \vspace{1pt}
    \end{minipage}

    \begin{minipage}[b]{\linewidth}
    \footnotesize
    \setlength{\tabcolsep}{4pt}
    \centering
    \begin{tabular}{c | c || P{1.55cm} P{1.55cm} P{1.55cm}}
        \hline
        Ball & Ramp & VelTrack & E2E & DRIS (Ours)\\
        \hline
        \multirow{3}{*}{Wiffle} & $R = 0.13\, m$ & 0\phantom{0} / \phantom{0}5 & 0\phantom{0} / \phantom{0}5 & \textbf{5\phantom{0} / \phantom{0}5} \\
                                & $R = 0.20\, m$ & 0\phantom{0} / \phantom{0}5 & 2\phantom{0} / \phantom{0}5 & \textbf{4\phantom{0} / \phantom{0}5} \\
                                & $R = 0.32\, m$ & 0\phantom{0} / \phantom{0}5 & 0\phantom{0} / \phantom{0}5 & \textbf{3\phantom{0} / \phantom{0}5} \\
        \hline
        \multirow{3}{*}{Rubber} & $R = 0.13\, m$ & 0\phantom{0} / \phantom{0}5 & 0\phantom{0} / \phantom{0}5 & \textbf{2\phantom{0} / \phantom{0}5} \\
                                & $R = 0.20\, m$ & 0\phantom{0} / \phantom{0}5 & 0\phantom{0} / \phantom{0}5 & \textbf{5\phantom{0} / \phantom{0}5} \\
                                & $R = 0.32\, m$ & 0\phantom{0} / \phantom{0}5 & 1\phantom{0} / \phantom{0}5 & \textbf{2\phantom{0} / \phantom{0}5} \\
        \hline
        \multirow{3}{*}{Ping-Pong} & $R = 0.13\, m$ & 0\phantom{0} / \phantom{0}5 & 1\phantom{0} / \phantom{0}5 & \textbf{4\phantom{0} / \phantom{0}5} \\
                                   & $R = 0.20\, m$ & 2\phantom{0} / \phantom{0}5 & 1\phantom{0} / \phantom{0}5 & \textbf{5\phantom{0} / \phantom{0}5} \\
                                   & $R = 0.32\, m$ & 0\phantom{0} / \phantom{0}5 & 0\phantom{0} / \phantom{0}5 & \textbf{4\phantom{0} / \phantom{0}5} \\
        \hline
        \multirow{3}{*}{Foam} & $R = 0.13\, m$ & 0\phantom{0} / \phantom{0}5 & 0\phantom{0} / \phantom{0}5 & \textbf{2\phantom{0} / \phantom{0}5} \\
                              & $R = 0.20\, m$ & 1\phantom{0} / \phantom{0}5 & 1\phantom{0} / \phantom{0}5 & \textbf{3\phantom{0} / \phantom{0}5} \\
                              & $R = 0.32\, m$ & 0\phantom{0} / \phantom{0}5 & \textbf{2\phantom{0} / \phantom{0}5} & \textbf{2\phantom{0} / \phantom{0}5} \\
        \hline
        \multicolumn{2}{c ||}{Total} & 3\phantom{0} / 60 & 8\phantom{0} / 60 & \textbf{41 / 60}\\
        \multicolumn{2}{c ||}{Success Rate} & 0.05 & 0.13 & \textbf{0.68}\\
        \hline
    \end{tabular}
    \end{minipage}
    
    \caption{
    \emph{Top:} Snapshots of our policy catching a ping-pong ball released from the ramp.
    \emph{Bottom:} Success rates for each policy across all combinations of ball types and ramp settings.
    }
    \label{fig:eval_real}
\end{figure}

\subsection{System Setup}

Fig.~\ref{fig:setup} shows the setup we used for real-world evaluation (top),
and the four test balls (wiffle, rubber, ping-pong, and foam) which differ in size, weight, and physical behavior (bottom).
The plate is 3D-printed and covered with a neoprene foam padding to protect the robot from damage during repeated trials.
Although the padding slightly reduces the impact force,
its surface is smoother than the underlying plastic and causes the ball to roll and accelerate more easily, thus preserving the difficulty of the task.

To ensure relatively consistent initial ball states for fair comparative evaluation across different policies,
we designed and 3D-printed a ramp with interchangeable lower sections and mounted it above the workspace.
The ball is released by rolling from the top of the ramp,
and the three interchangeable sections (with radii $R = 0.13$, $0.20$, and $0.32\, m$) induce different release speeds,
allowing evaluation under varying initial ball states.

\subsection{Quantitative Evaluation}

We deployed a single DRIS-trained policy (with $N=200$) on the real robot across all evaluation trials and compared its performance against two baselines:
1) \emph{VelTrack}: a hand-crafted policy that moves the plate horizontally to follow the ball based on its estimated velocity after the first impact,
with the plate oriented to face the ball's incoming direction at the moment of first impact;
2) \emph{E2E}: similar to the simulation evaluation in Sec.~\ref{sec:expsim}, this is a policy trained end-to-end with a single ball in simulation and directly deployed on the real robot.
The controller gains, i.e., $\bm{K}$ and $\bm{D}$ in Eq.~\eqref{eq:control}, 
were identified on the real robot by matching the controller's behavior to that in simulation.
Aside from this gain identification,
no additional adaptation or fine-tuning was performed.

For each policy, and for every combination of ball type and ramp,
we conducted $5$ trials and evaluated performance based on the number of successful catches, as summarized in Fig.~\ref{fig:eval_real}.
A trial is considered a successful catch if the robot keeps the ball on the plate for at least $10$ seconds after it begins moving.
The VelTrack baseline rarely succeeded,
as its response was not sufficiently reactive,
particularly under inaccurate ball velocity estimates.
Its only $3$ successful trials occurred when the ball was accidentally trapped in the gap between the plate and the finger.
The E2E baseline, although producing reasonable motions to contact the ball,
was highly sensitive to observation noise, 
leading to jerky robot movements that frequently knocked the ball off the plate or hit the robot's power limit.
In contrast, our DRIS-trained policy achieved a $68\%$ success rate when deployed zero-shot, and successfully caught all tested balls.
While the system maintains high performance across diverse conditions, 
the few remaining failure cases typically occur when the initial contact cannot sufficiently dissipate the ball's momentum or redirect the ball unfavorably due to tracking latency or noise.
In these instances, the robot may occasionally over-stretch while attempting to recover the high-speed trajectory.

\begin{figure*}[t]
    \centering
    \begin{tikzpicture}
        \node[anchor=south west,inner sep=0] at (0,0){\includegraphics[width=\linewidth]{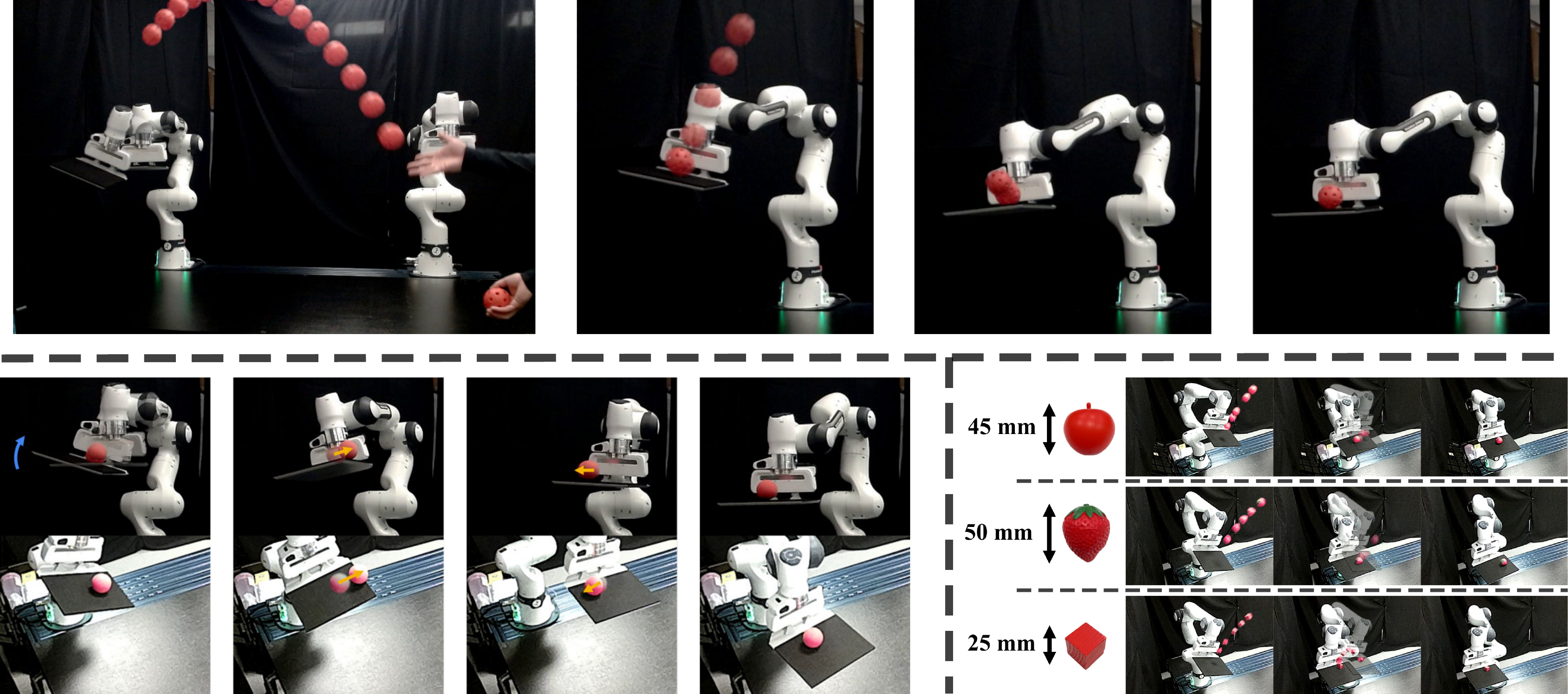}};
        \node[anchor=west, align=left, text=white, font=\scriptsize] at (0, 3.4) {View 1};
        \node[anchor=west, align=left, text=white, font=\scriptsize] at (0, 0.2) {View 2};
    \end{tikzpicture}
    \caption{
        Our policy successfully catches a wiffle ball thrown by a human (top); 
        balances a foam ball (bottom left), where a hard-coded motion rotates the plate by $30^{\circ}$ (blue arrow) to initiate the ball's rolling in the first frame; 
        and catches irregularly shaped objects (bottom right).
    }
    \label{fig:qual}
\end{figure*}

\subsection{Qualitative Evaluation}

We conducted several qualitative evaluations to further assess the learned policy.
First, as showcased in Fig.~\ref{fig:qual} (top),
the policy successfully catches all test balls thrown by a human, 
despite the increased uncertainty and inconsistency of human throws.
Second, as illustrated in Fig.~\ref{fig:qual} (bottom left),
although trained for reactive catching, 
the policy generalizes directly to a pure balancing task.
Third, as shown in Fig.~\ref{fig:qual} (bottom right),
the policy was able to catch irregularly shaped objects, including a lightweight toy apple, a heavier toy strawberry, and a wooden cube.
The wooden cube is caught only occasionally due to its highly unpredictable impact dynamics, 
which makes reliable catching difficult even for humans.

%% file: includes/conclusion.tex
\section{Conclusion and Limitations}
\label{sec:conclusion}

In this work, we propose a DRIS-based learning strategy that enables robust policy learning in simulation and improved zero-shot sim-to-real transfer for dexterous manipulation under severe physical uncertainty.
By simultaneously propagating multiple randomized instances instead of one, DRIS better approximates state evolution under diverse physical conditions.
Combined with a task-specific policy design, this enables stable reactive catching without passive mechanical aids.
Extensive experiments show substantially improved robustness and sim-to-real performance over conventional baselines.

\emph{Limitations.}
While the proposed DRIS-based strategy is promising, we acknowledge several limitations to be addressed in future work.
First, DRIS requires an encoder to map states of varying instances to a fixed-dimensional latent representation; 
while this is efficient for low-dimensional states, tasks involving high-dimensional visual or geometric inputs may require larger models and substantial pre-training. 
Second, effective DRIS training relies on efficient parallel propagation of multiple instances,
which may be computationally expensive or impractical without suitable simulator support. 
This is particularly relevant when scaling to complex multi-body contact settings involving multiple mutual interactions, where the computational overhead of resolving simultaneous collision manifolds across the entire DRIS can become a significant bottleneck.
Finally, since all instances share a single action, excessive divergence among instance states can destabilize learning, placing a practical limit on the amount of variation introduced.
As noted in~\cite{schaal1993open}, preventing such divergence is easier in dissipative systems. 
While our method thrives in ``energy-dissipating'' tasks like catching, 
it may struggle in ``energy-injecting'' tasks like throwing. 
In such expansionary dynamics, physical variations are rapidly amplified, making it difficult for a shared action to reconcile the diverging states. 
Future work using adaptive or curriculum-based instance randomization during training might mitigate this challenge.

%% file: includes/appendix_alg.tex
\clearpage
\setcounter{section}{0}
\renewcommand{\thesection}{Appendix \Alph{section}}
\section{Implementation Details}
\label{appendix:alg}

The learning pipeline and task simulation were executed on a single NVIDIA GeForce RTX~3060 GPU with 12~GB of memory.
The actions are generated at 20~Hz, corresponding to a simulation step of 1/20~second.
The plate size (half-length) is set to $l_e = 12 \, cm$.
The domain-randomized parameters were uniformly sampled from the following ranges: 
ball radius $\left[2, \, 4\right] cm$; static and dynamic friction coefficients $\left[0, \, 0.1 \right]$; and restitution coefficient $\left[0.4, \, 0.7 \right]$.
The latent feature dimension of encoded DRIS was set to $d_z = 64$, and the decay coefficient in the reward of Eq.~\eqref{eq:reward} was $\eta=0.25$.

When deploying the trained policy in the real world,
the ball's 3D trajectory is estimated using two cameras at $80$~FPS through image-based ball detection (e.g., color segmentation followed by contour-based circle fitting), triangulation, and parabolic curve fitting.
The ball's velocity is estimated via ordinary least squares (OLS) by fitting a line to position measurements over a $0.1$-second sliding window.
The trained policies were loaded on a 3.4 GHz AMD Ryzen 9 5950X CPU for real-time inference.

%% file: includes/appendix.tex
\section{Theoretical Analysis of DRIS}
\label{appendix:math}

In this appendix, we provide a theoretical analysis underlying Theorem~\ref{theorem:dris}.
Specifically, we contrast the conventional Domain Randomization (DR) with our proposed DRIS under a unified problem setup (~\ref{appendix:a});
we interpret DRIS as an exact particle approximation for the system belief propagation (~\ref{appendix:b});
then, we analyze to show that DRIS yields a more stable optimization (~\ref{appendix:c}), more robust policies (~\ref{appendix:d}), and improved sim-to-real generalization (~\ref{appendix:e}),
followed by a brief discussion on the connection to our practical implementation (\ref{appendix:f}).

\subsection{Unified Problem Setup}
\label{appendix:a}

To ensure this appendix is self-contained, we briefly recapitulate the problem setup and notations from Sec.~\ref{sec:method} in the main text.
The manipulation problem is defined over an observable state space $\mathcal{S}$ and an action space $\mathcal{A}$.
We denote a state by $\bm{s} \in \mathcal{S}$ and an action by $\bm{a} \in \mathcal{A}$.
The system state evolves according to a deterministic transition function $\bm{s}_{t+1} = f\left( \bm{s}_t, \, \bm{a}_t, \, \bm{c} \right)$,
conditioned on the physical parameters $\bm{c} \in \mathcal{C}$ (e.g., friction, mass).
The parameters $\bm{c}$ are unknown during execution but are assumed to be distributed according to a fixed prior $p\left(\bm{c} \right)$ (which is simplified to a uniform distribution in our implementation).

\emph{Ensemble State with Shared Actions.}
In the proposed \textbf{Domain-Randomized Instance Set (DRIS)} framework, 
during policy learning,
we sample a set of $N$ i.i.d. physical instances drawn from the prior, 
denoted $\hat{\mathcal{C}} = \left\{ \bm{c}^{(1)}, \dots, \bm{c}^{(N)} \right\}$ where each $\bm{c}^{(i)} \sim p(\bm{c})$,
The ensemble state of these $N$ parallel instances is denoted as $\mathcal{S}_t = \left\{ \bm{s}_t^{(1)}, \dots, \bm{s}_t^{(N)} \right\} \in \mathcal{S}^N$.
At each time step $t$, the policy $\pi_{\bm{\theta}}$ (parameterized by $\bm{\theta}$) generates a single consensus action $\bm{a}_t = \pi_{\bm{\theta}} \left( \mathcal{S}_t \right)$ that is applied to all $N$ parallel instances simultaneously.
This introduces an explicit coupling:
The state transition of instance $i$ depends on the states of all other instances via the shared action:
\begin{equation}
    \bm{s}_{t+1}^{(i)} = f\left( \bm{s}_t^{(i)}, \, \pi_{\bm{\theta}}\left(\mathcal{S}_t \right), \, \bm{c}^{(i)} \right), \quad \forall i \in \{1, \dots, N\}
\end{equation}

\emph{The Global Objective.}
With all instances starting from the same initial state $\bm{s}_0 \sim p\left(\bm{s}_0\right)$,
the controlled trajectory of the $i$-th instance is $\tau_{\bm{c}^{(i)}} = \left(\bm{s}_0, \bm{a}_0, \bm{s}_1^{(i)}, \bm{a}_1, \dots, \bm{s}_T^{(i)} \right)$.
We denote $\tau_{\bm{c}^{(i)}}$ to emphasize that the trajectory is implicitly conditioned on the specific realization of the physical parameter $\bm{c}^{(i)}$.
The \emph{Discounted Cumulative Reward} (Return) of a single trajectory is then defined as:
\begin{equation}
\label{eq:R}
   R\left( \tau_{\bm{c}^{(i)}} \right) \coloneqq \sum_{t=0}^{T-1} \gamma^t \, r\left( \bm{s}^{(i)}_t, \, \bm{a}_t \right)
\end{equation}
The global optimization objective with respect to the policy parameter $\bm{\theta}$ is to maximize the expected average return over all instance trajectories:  
\begin{equation}
    \mathcal{J}_N \left( \bm{\theta} \right) \coloneqq \mathbb{E}_{\bm{s}_0 \sim p\left( \bm{s}_0 \right), \bm{c}^{(i)} \sim p\left( \bm{c} \right)} \left[ \frac{1}{N} \sum_{i=1}^N R\left( \tau_{\bm{c}^{(i)}} \right) \right]
\end{equation}

\textbf{Relationship to Conventional Domain Randomization (DR).}
By setting the DRIS size to $N = 1$, the ensemble state collapses to a singleton $\mathcal{S}_t = \left\{ \bm{s}_t \right\}$ and the unified objective $\mathcal{J}_N$ reduces to:
\begin{equation}
    \mathcal{J}_1 \left( \bm{\theta} \right) = \mathbb{E}_{\bm{s}_0 \sim p\left( \bm{s}_0 \right), \bm{c} \sim p\left( \bm{c} \right)} \left[ R\left(\tau_{\bm{c}} \right) \right]
\end{equation}
This recovers the conventional DR objective, where the expectation is estimated via Monte Carlo sampling of a single physics parameter per episode.

\subsection{Interpretation: Exact Particle Propagation}
\label{appendix:b}

We prove that the simulation step in DRIS corresponds to the exact propagation of a particle approximation of the system's belief state.

\textbf{Continuous Belief Dynamics.}
Let $b_t \left(\bm{s}, \bm{c} \right)$ be the joint probability density (belief) of the state and physics parameters at time $t$. 
Given a deterministic action $\bm{a}_t$, 
the belief evolves according to the \textbf{Liouville Equation} (the continuity equation for probability mass):
\begin{equation}
\label{eq:liouville}
\begin{aligned}
    b_{t+1}\left(\bm{s}', \bm{c}'\right)
    =
    \int_{\mathcal{S}}
    \int_{\mathcal{C}}
    & \delta \Big(
        \bm{s}' - f\left(\bm{s}, \bm{a}_t, \bm{c}\right)
    \Big)
    \\
    &
    \cdot
    \delta \left(\bm{c}' - \bm{c}\right)
    \, b_t \left(\bm{s}, \bm{c}\right)
    \, d\bm{c} \, d\bm{s}
\end{aligned}
\end{equation}
Here, $\delta\left(\cdot \right)$ is the Dirac delta function. 
The term $\delta\left(\bm{c}' - \bm{c} \right)$ enforces that physics parameters are static during propagation.

\begin{proposition}[Exact Particle Propagation]
Let the initial belief be approximated by an empirical measure $\hat{b}_{t, N}\left(\bm{s}, \bm{c}\right) = \frac{1}{N} \sum_{i=1}^N \delta (\bm{s} - \bm{s}_t^{(i)} ) \delta(\bm{c} - \bm{c}^{(i)})$. 
The time evolution of this measure under the Liouville Equation is exactly the empirical measure constructed from the updated particles $\bm{s}_{t+1}^{(i)} = f\left(\bm{s}_t^{(i)}, \, \bm{a}_t, \, \bm{c}^{(i)} \right)$.
\end{proposition}

\begin{proof}
Substituting $\hat{b}_{t,N}$ into Eq.~\eqref{eq:liouville} and utilizing the sifting property of the Dirac delta $\int f(x)\delta(x-x_0)dx = f(x_0)$:
\begin{equation}
\begin{aligned}
    \hat{b}_{t+1}\left(\bm{s}', \bm{c}'\right) 
    &= \iint \delta\Big(\bm{s}' - f\left(\bm{s}, \bm{a}_t, \bm{c}\right) \Big) \, \delta(\bm{c}' - \bm{c}) \\
    &\quad \quad \left[ \frac{1}{N} \sum_{i=1}^N \delta(\bm{s} - \bm{s}_t^{(i)}) \delta(\bm{c} - \bm{c}^{(i)}) \right] d\bm{s} \, d\bm{c} \\
    &= \frac{1}{N} \sum_{i=1}^N \iint \delta\Big(\bm{s}' - f\left(\bm{s}, \bm{a}_t, \bm{c}\right) \Big) \, \delta(\bm{c}' - \bm{c})\\
    & \quad \quad \delta(\bm{s} - \bm{s}_t^{(i)}) \, \delta(\bm{c} - \bm{c}^{(i)}) \, d\bm{s} \, d\bm{c} \\
    &= \frac{1}{N} \sum_{i=1}^N \delta\Big( \bm{s}' - f\left(\bm{s}_t^{(i)}, \bm{a}_t, \bm{c}^{(i)}\right) \Big) \, \delta\left(\bm{c}' - \bm{c}^{(i)}\right)
\end{aligned}
\end{equation}
This resulting distribution corresponds precisely to the discrete set of updated simulation states. 
Thus, DRIS is mathematically equivalent to propagating the belief state through the exact system dynamics.
\end{proof}

\subsection{Gradient Variance Reduction}
\label{appendix:c}

We analyze the variance of the gradient estimator $\nabla_{\bm{\theta}} \mathcal{J}_N \left(\bm{\theta}\right)$.
In particular, assuming the policy is permutation invariant (i.e., the ordering of instances in the DRIS does not affect the policy output),
we show that DRIS yields a reduced gradient variance compared to conventional DR.

\begin{definition}[Permutation Invariant Policy]
A policy $\pi_{\bm{\theta}}$ is permutation invariant if for any permutation $\phi$ of indices $\{1, \dots, N\}$ and any ensemble state $\left\{ \bm{s}^{(i)}, \dots, \bm{s}_{(N)} \right\} \in \mathcal{S}^N$:
\begin{equation}
    \pi_{\bm{\theta}}\left( \left\{ \bm{s}^{\left(\phi(1)\right)}, \dots, \bm{s}^{\left(\phi(N)\right)} \right\} \right) = \pi_{\bm{\theta}}\left( \left\{ \bm{s}^{(1)}, \dots, \bm{s}^{(N)} \right\} \right)
\end{equation}
\end{definition}

\begin{lemma}[Gradient Exchangeability]
\label{lemma:grad_exchange}
Let $\bm{g}_i \triangleq \nabla_{\bm{\theta}} R\left(\tau_{\bm{c}^{(i)}} \right)$ denote the gradient obtained from the $i$-th instance. 
If $\pi_{\bm{\theta}}$ is permutation invariant,
the gradients $\bm{g}_1, \dots, \bm{g}_N$ are exchangeable random variables. 
Consequently, they share a common mean $\bm{\mu}$, a common covariance matrix $\bm{\Sigma}$ and a common cross-covariance matrix $\bm{\Sigma}_{\text{cross}}$:
\begin{equation}
    \text{Cov}\left(\bm{g}_i\right) = \bm{\Sigma}, \quad \text{Cov}\left(\bm{g}_i, \bm{g}_j\right) = \bm{\Sigma}_{\text{cross}} \quad \forall i \ne j
\end{equation}
\end{lemma}

\begin{proof}
Let $\bm{G} = \left(\bm{g}_1, \dots, \bm{g}_N\right)$ be the joint vector of gradients. 
Due to the i.i.d. sampling of physics parameters and the permutation invariance of $\pi_{\bm{\theta}}$, the joint distribution of $\bm{G}$ is invariant under any permutation $\phi$: $p\left(\bm{g}_1, \dots, \bm{g}_N\right) = p\left(\bm{g}_{\phi(1)}, \dots, \bm{g}_{\phi(N)}\right)$.
Hence, $\bm{g}_1, \dots, \bm{g}_N$ are exchangeable.
Exchangeability implies identical marginals, so $\mathbb{E}[\bm{g}_i] = \bm{\mu}$ and $\text{Cov}(\bm{g}_i) = \mathbb{E}[(\bm{g}_i - \bm{\mu})(\bm{g}_i - \bm{\mu})^T] = \bm{\Sigma}$ for all $i$. 
Moreover, exchangeability implies that all unordered pairs $(\bm{g}_1, \bm{g}_j)$, $i\neq j$, have identical joint distributions, and thus identical cross-covariance $\text{Cov}(\bm{g}_i, \bm{g}_j) = \mathbb{E}[(\bm{g}_i - \bm{\mu})(\bm{g}_j - \bm{\mu})^T] = \bm{\Sigma}_{\text{cross}}$ for all $i \neq j$.
\end{proof}

\begin{proposition}[Variance Reduction]
Let the scalar total variance be the trace of the covariance matrix: $\sigma^2 \triangleq \text{Tr}\left(\bm{\Sigma}\right)$,
and the scalar correlation coefficient be $\rho \triangleq \text{Tr}\left(\bm{\Sigma}_{\text{cross}}\right) / \sigma^2$. 
Then the total variance of the DRIS gradient estimator $\hat{\bm{g}}_{DRIS} = \frac{1}{N}\sum \bm{g}_i$ is given by:
\begin{equation}
    \text{Tr}\left(\text{Cov}\left(\hat{\bm{g}}_{DRIS}\right)\right) = \sigma^2 \left( \rho + \frac{1-\rho}{N} \right)
\end{equation}
Consequently, DRIS ($N > 1$) achieves strict variance reduction compared to conventional DR ($N=1)$ if and only if $\rho < 1$ (i.e., when the aggregate cross-covariance is strictly smaller than the marginal covariance).
\end{proposition}

\begin{proof}
Using the bilinearity of covariance and the exchangeability of the gradients:
\begin{equation}
\begin{aligned}
    \text{Cov}\left( \hat{\bm{g}}_{DRIS} \right) &= \text{Cov}\left( \frac{1}{N} \sum_{i=1}^N \bm{g}_i \right) \\
    &= \frac{1}{N^2} \left[ \sum_{i=1}^N \text{Cov}\left(\bm{g}_i\right) + \sum_{i \neq j} \text{Cov}\left(\bm{g}_i, \bm{g}_j\right) \right]
\end{aligned}
\end{equation}
Substituting the covariance and cross-covariance matrices defined in Lemma.~\ref{lemma:grad_exchange}:
\begin{equation}
\begin{aligned}
    \text{Cov}\left( \hat{\bm{g}}_{DRIS} \right) 
    &= \frac{1}{N^2} \left[ N \bm{\Sigma} + N(N-1) \bm{\Sigma}_{\text{cross}} \right] \\
    &= \frac{1}{N} \bm{\Sigma} + \frac{N-1}{N} \bm{\Sigma}_{\text{cross}}
\end{aligned}
\end{equation}
Taking the trace (a linear operator) to obtain the scalar total variance:
\begin{equation}
\begin{aligned}
    \text{Tr}\left(\text{Cov}\left(\hat{\bm{g}}_{DRIS}\right)\right) 
    &= \frac{1}{N} \text{Tr}\left(\bm{\Sigma}\right) + \left( 1 - \frac{1}{N} \right) \text{Tr}\left(\bm{\Sigma}_{\text{cross}}\right) \\
    &= \frac{\sigma^2}{N} + \left( 1 - \frac{1}{N} \right) \rho \sigma^2 \\
    &= \sigma^2 \left( \rho + \frac{1-\rho}{N} \right)
\end{aligned}
\end{equation}
In practice, independently sampled physics parameters introduce instance-specific perturbations in rollout trajectories,
so gradients across instances are typically not perfectly aligned and $\rho < 1$ is expected,
yielding strict variance reduction for DRIS ($N > 1$) compared to conventional DR ($N = 1$).
\end{proof}

\textbf{Implication:}
The result above implies that as the DRIS size increases ($N \to \infty$),
the estimator's total variance decreases from the conventional DR baseline $\sigma^2$ to a non-zero floor $\rho \sigma^2$.
Intuitively,
increasing $N$ averages out the instance-specific (uncorrelated) component of gradient variability;
this stabilizes the optimization by suppressing the high-variance noise from single-instance rollouts.

\subsection{Robustness Analysis}
\label{appendix:d}

We analyze how the DRIS objective implicitly promotes robustness against physical uncertainty.
While conventional DR and DRIS optimize the same theoretical objective in expectation, their empirical optimization landscapes differ fundamentally.

In our problem setup, we can define the empirical objective as the average return over the $N$ rollout trajectories of DRIS.
Since physics parameters $\bm{c}^{(i)}$ are sampled i.i.d. from $p\left(\bm{c}\right)$, the rollouts are conditionally i.i.d. across instances given the shared action sequence $\left\{\bm{a}_t\right\}_{t=0}^{T-1}$:
\begin{equation}
    \hat{\mathcal{J}}_N\left(\bm{\theta}\right) \coloneqq \frac{1}{N} \sum_{i=1}^N R\left( \tau_{\bm{c}^{(i)}} \right) = \frac{1}{N} \sum_{i=1}^N \sum_{t=0}^{T-1} \gamma^t \, r\left(\bm{s}_t^{(i)}, \, \bm{a}_t\right)
\end{equation}

We analyze the optimization landscape by decomposing this return into its local contributions at a single time step $t$,
and define empirical local cost (negative reward) as $C_t(\bm{s}) \triangleq -r\left(\bm{s}, \bm{a}_t \right)$.
The global maximization is equivalent to minimizing the discounted sum of empirical local costs:
\begin{equation}
    \hat{\mathcal{J}}_N(\bm{\theta}) = - \sum_{t=0}^{T-1} \gamma^t \, \underbrace{\left( \frac{1}{N} \sum_{i=1}^N C_t(\bm{s}_t^{(i)}) \right)}_{\hat{J}_{N,t}}
\label{eq:local_sum}
\end{equation}
This decomposition simplifies the analysis:
by examining the local cost $\hat{J}_{N,t}$ conditioned on the shared action $\bm{a}_t$,
we can determine how $N$ affects the signal-to-noise ratio of the gradient at each step.

\begin{assumption}[Local Geometry]
\label{assumption:local_cost}
We assume $C_t(\bm{s})$ is twice differentiable and locally convex w.r.t. the state $\bm{s}$ near the mean $\bm{\mu}_{t} \triangleq \mathbb{E}[\bm{s}_t \mid \bm{a}_t]$. 
We approximate $C_t(\bm{s}_t)$ using a second-order Taylor expansion:
\begin{equation}
\label{eq:taylor}
\begin{aligned}
    C_t\left(\bm{s}\right) &\approx C_t(\bm{\mu}_{t}) + \bm{g}_t^\top (\bm{s} - \bm{\mu}_{t}) + \frac{1}{2} (\bm{s} - \bm{\mu}_{t})^\top \mathbf{H}_{t} (\bm{s} - \bm{\mu}_{t})
\end{aligned}
\end{equation}
where $\bm{g}_t = \nabla_{\bm{s}} C_t\left(\bm{\mu}_{t}\right)$ and $ \mathbf{H}_{t} = \nabla^2_{\bm{s}} C_t(\bm{\mu}_t) \succeq 0$ are the Jacobian and Hessian, respectively.
\end{assumption}

Substituting this expansion into $\hat{J}_{N,t}$ and using the cyclic property of the trace operator $\text{Tr}\left(\bm{x}^\top \mathbf{A} \bm{x}\right) = \text{Tr}\left(\mathbf{A}\bm{x}\bm{x}^\top\right)$, we derive the decomposition:
\begin{equation}
\begin{aligned}
    \hat{J}_{N,t} &= \frac{1}{N} \sum_{i=1}^N C_t\left(\bm{s}_t^{(i)}\right) \\
    &\approx \frac{1}{N} \sum_{i=1}^N \Big[ C_t(\bm{\mu}_t) + \bm{g}_t^\top (\bm{s}_t^{(i)} - \bm{\mu}_t)\\
    & \quad + \frac{1}{2} \text{Tr}\left( \mathbf{H}_t (\bm{s}_t^{(i)} - \bm{\mu}_t)(\bm{s}_t^{(i)} - \bm{\mu}_t)^\top \right) \Big] \\
    & = C_t(\bm{\mu}_t) + \bm{g}_t^\top \left( \frac{1}{N} \sum_{i=1}^N (\bm{s}_t^{(i)} - \bm{\mu}_t) \right)\\
    & \quad + \frac{1}{2} \text{Tr}\left( \mathbf{H}_t \left( \frac{1}{N} \sum_{i=1}^N (\bm{s}_t^{(i)} - \bm{\mu}_t)(\bm{s}_t^{(i)} - \bm{\mu}_t)^\top \right) \right)\\
    &= \underbrace{C_t(\bm{\mu}_t)}_{\text{Nominal}} + \underbrace{\bm{g}_t^\top \bar{\bm{\delta}}_{N,t}}_{\text{Linear Noise}} + \underbrace{\frac{1}{2} \mathrm{Tr}\left(\mathbf{H}_t \hat{\bm{\Sigma}}_{N,t}\right)}_{\text{Robustness Signal}}
\end{aligned}
\end{equation}
where $\bar{\bm{\delta}}_{N,t} = \frac{1}{N}\sum_{i=1}^N\left(\bm{s}^{(i)}_t - \bm{\mu}_t\right)$ is the sample mean error;
$\hat{\bm{\Sigma}}_{N,t} =  \frac{1}{N} \sum_{i=1}^N (\bm{s}_t^{(i)} - \bm{\mu}_t)(\bm{s}_t^{(i)} - \bm{\mu}_t)^\top$ is the empirical second moment about the true mean, satisfying $\mathbb{E}\left[\hat{\bm{\Sigma}}_{N,t}\right] = \bm{\Sigma}_t$.

\begin{theorem}[Asymptotic Unmasking]
As the DRIS size $N$ increases, the linear noise is suppressed while the robustness signal concentrates. 
The empirical local cost converges in probability to a regularized objective:
\begin{equation}
    \hat{J}_{N,t} \xrightarrow{P} C_t\left(\bm{\mu}_t\right) + \frac{1}{2}\mathrm{Tr}\left(\mathbf{H}_t \bm{\Sigma}_t\right) \quad \text{as } N \to \infty.
\end{equation}
\end{theorem}

\begin{proof}
By the Weak Law of Large Numbers, 
the error $\bar{\bm{\delta}}_{N,t} \xrightarrow{P} \bm{0}$, causing the linear noise term $\bm{g}_t^\top \bar{\bm{\delta}}_{N,t}$ to vanish. 
Similarly, $\hat{\bm{\Sigma}}_{N,t} \xrightarrow{P} \bm{\Sigma}_t$. 
By the Continuous Mapping Theorem, 
$\frac{1}{2}\mathrm{Tr}\left(\mathbf{H}_t \hat{\bm{\Sigma}}_{N,t}\right) \xrightarrow{P} \frac{1}{2}\mathrm{Tr}\left(\mathbf{H}_t \bm{\Sigma}_t\right)$. 
Due to local convexity ($\mathbf{H}_t \succeq 0$), this limit is strictly non-negative, acting as a penalty on aleatoric (physics-induced) variance.
\end{proof}

\textbf{Implication.}
Substituting this local result back into Eq.~\eqref{eq:local_sum} proves that DRIS improves the Signal-to-Noise Ratio (SNR) of the global optimization.
\begin{itemize}
    \item \textbf{At $N=1$ (Conventional DR):} 
    The objective is dominated by linear noise $\sum \gamma^t \bm{g}_t^\top \bar{\bm{\delta}}_{1,t}$. 
    The optimizer ``chases'' random first-order fluctuations, which can obscure the variance penalty.
    \item \textbf{At $N \gg 1$ (DRIS):} 
    The linear noise is suppressed by $\mathcal{O}\left(1 / \sqrt{N}\right)$. 
    The global objective behaves as:
    \begin{equation}
        \hat{\mathcal{J}}_N\left(\bm{\theta}\right) \approx - \sum_{t=0}^{T-1} \gamma^t \left[ C_t(\bm{\mu}_t) + \frac{1}{2}\mathrm{Tr}\left(\mathbf{H}_t \bm{\Sigma}_t\right) \right]
    \end{equation}
    This steers the policy to update in directions that reduce the physics-induced variance $\bm{\Sigma}_t$.
\end{itemize}
As a result, the learned policy via DRIS favors state-action regions that are invariant (i.e., robust) to physical uncertainty.

\subsection{Sim-to-Real Generalization}
\label{appendix:e}

We provide a theoretical analysis of the sim-to-real transfer capability of DRIS.
We explicitly account for the distribution shift between the physics in simulation and the true physics of the real world.

Recall that given a fixed initial state $\bm{s}_0$, for physics parameter $\bm{c}$, we let $\tau_{\bm{c}}$ denote the trajectory induced by policy $\pi_{\bm{\theta}}$.
As defined earlier in Eq.~\eqref{eq:R}, the discounted return of the trajectory is $R\left(\tau_{\bm{c}}\right)$.
We equivalently define the trajectory cost as the negative return:
\begin{equation}
    \mathcal{L}(\bm{\theta}, \bm{c})
    \coloneqq
    - R\left(\tau_{\bm{c}}\right)
    =
    - \sum_{t=0}^{T-1} \gamma^t \, r\left(\bm{s}_t, \bm{a}_t\right)
\end{equation}
Maximizing return is therefore equivalent to minimizing trajectory cost.
If rewards are bounded as $\lvert r(\bm{s},\bm{a})\rvert \le r_{\max}$, then
\begin{equation}
    \left\lvert \mathcal{L}(\bm{\theta},\bm{c}) \right\rvert
    \le
    \frac{r_{\max}}{1-\gamma}
    \eqqcolon B
\end{equation}

\textbf{Source vs. Target Distributions.}
We distinguish between two distributions over the physics parameters $\mathcal{C}$:
\begin{itemize}
    \item \textbf{Source Distribution} with density $p_S(\cdot)$: The distribution used in the simulator for learning (e.g., uniform).
    \item \textbf{Target Distribution} with density $p_T(\cdot)$: The unknown true distribution of physical parameters in the real world.
\end{itemize}
We seek to minimize the expected cost in the real world:
\begin{equation}
    \mathcal{J}_T\left(\bm{\theta}\right) \coloneqq \mathbb{E}_{\bm{c} \sim p_T(\bm{c})} \left[\mathcal{L}(\bm{\theta}, \bm{c})\right]
\end{equation}
The DRIS algorithm minimizes the empirical cost drawn from $p_S(\cdot)$, averaged over the $N$ instances inside DRIS:
\begin{equation}
    \hat{\mathcal{J}}_N\left(\bm{\theta}\right) \coloneqq \frac{1}{N} \sum_{i=1}^N \mathcal{L}(\theta, \bm{c}^{(i)}), \quad \text{where } \bm{c}^{(i)} \overset{\text{i.i.d.}}{\sim} p_S\left(\bm{c}\right)
\end{equation}
Let $\mathcal{J}_S\left(\bm{\theta}\right) = \mathbb{E}_{\bm{c} \sim p_S(\bm{c})}\left[\mathcal{L}\left(\bm{\theta}, \bm{c} \right) \right]$ be the expected cost in simulation.

\begin{definition}[Rademacher Complexity]
To quantify the capacity of the policy parameter class $\bm{\Theta}$ to fit random fluctuations in the simulation, 
we introduce the \textbf{empirical Rademacher Complexity} of the induced (trajectory) cost function class $\mathcal{L}_{\bm{\Theta}} = \{ \mathcal{L}(\bm{\theta}, \, \cdot) : \bm{\theta} \in \bm{\Theta} \}$.
Given a sample $\hat{\mathcal{C}} = \{\bm{c}^{(1)}, \dots, \bm{c}^{(N)}\}$, it is defined as:
\begin{equation}
    \hat{\mathfrak{R}}_{\hat{\mathcal{C}}}\left(\mathcal{L}_{\bm{\Theta}}\right) \coloneqq \mathbb{E}_{\bm{\sigma}} \left[ \sup_{\bm{\theta} \in \bm{\Theta}} \frac{1}{N} \sum_{i=1}^N \sigma_i \, \mathcal{L}\left(\bm{\theta}, \bm{c}^{(i)} \right) \right],
\end{equation}
where $\bm{\sigma} = \left(\sigma_1, \dots, \sigma_N \right)$ are independent Rademacher variables with $\mathbb{P}(\sigma_i = 1) = \mathbb{P}(\sigma_i = -1) = \frac{1}{2}$.
The \textbf{expected Rademacher Complexity} is then defined by averaging over the sampling of $\hat{\mathcal{C}}$:
\begin{equation}
\mathfrak{R}_N\left(\mathcal{L}_{\bm{\Theta}} \right) \coloneqq \mathbb{E}_{\hat{\mathcal{C}} \sim p_S(\bm{c})^N} \left[\hat{\mathfrak{R}}_{\hat{\mathcal{C}}}(\mathcal{L}_{\bm{\Theta}})\right].
\end{equation}
\end{definition}

\begin{theorem}[Sim-to-Real Transfer Bound]
For any $\delta \in (0, 1)$, 
with probability at least $1-\delta$ over the draw of DRIS $\hat{\mathcal{C}} \sim p_S(\bm{c})^N$, 
the following bound holds for all policy parameter $\bm{\theta} \in \bm{\Theta}$:
\begin{equation}
\begin{aligned}
    \mathcal{J}_T\left(\bm{\theta}\right) \le \hat{\mathcal{J}}_N\left(\bm{\theta}\right) &+ \underbrace{2\mathfrak{R}_N\left(\mathcal{L}_{\bm{\Theta}}\right) + 2B\sqrt{\frac{\ln(1/\delta)}{2N}}}_{\text{Generalization Gap (Reducible via } N)}\\
    &+ \underbrace{d_{\mathcal{L}_{\bm{\Theta}}}(p_S, \, p_T)}_{\text{Physics Mismatch (Irreducible)}}
\end{aligned}
\end{equation}
where $d_{\mathcal{L}_{\bm{\Theta}}}(p_S, p_T)$ is the Integral Probability Metric (IPM) measuring the discrepancy between source and target distributions over physics parameter (i.e., $\mathcal{C}$) with respect to the cost class induced by the policy family.
\begin{equation}
\label{eq:ipm}
    d_{\mathcal{L}_{\bm{\Theta}}}(p_S, p_T) \coloneqq 
    \sup_{\bm{\theta} \in \bm{\Theta}} \left| \int_{\mathcal{C}} \mathcal{L}(\bm{\theta}, \bm{c})\, \left( p_T(\bm{c}) - p_S(\bm{c}) \right)\, d\bm{c} \right|
\end{equation}
\end{theorem}

\begin{proof}
Applying the triangle inequality, we decompose the sim-to-real error into two distinct components: 
an \textit{estimation error} (finite sampling) and a \textit{transfer error} (distribution shift):
\begin{equation}
\begin{aligned}
    \mathcal{J}_T\left(\bm{\theta} \right) &= \hat{\mathcal{J}}_N(\bm{\theta}) + \left(\mathcal{J}_S(\bm{\theta}) - \hat{\mathcal{J}}_N(\bm{\theta})\right) + \Big(\mathcal{J}_T(\bm{\theta}) - \mathcal{J}_S(\bm{\theta}) \Big) \\
    &\le \hat{\mathcal{J}}_N(\bm{\theta}) + \underbrace{\left|\mathcal{J}_S(\bm{\theta}) - \hat{\mathcal{J}}_N(\bm{\theta}) \right|}_{\text{Estimation Error (Step 1)}} + \underbrace{\Big| \mathcal{J}_T(\bm{\theta}) - \mathcal{J}_S(\bm{\theta}) \Big|}_{\text{Transfer Error (Step 2)}}
\end{aligned}
\end{equation}

\textbf{Step 1: Bounding the Estimation Error.}
We seek to bound $\sup_{\bm{\theta} \in \bm{\Theta}} \left| \mathcal{J}_S(\bm{\theta}) - \hat{\mathcal{J}}_N (\bm{\theta}) \right|$.
Let us define a deviation
\begin{equation}
    \Phi(\hat{\mathcal{C}}) = \sup_{\bm{\theta} \in \bm{\Theta}} \left( \mathcal{J}_S(\bm{\theta}) - \hat{\mathcal{J}}_N (\bm{\theta}) \right)
\end{equation}
Consider two sample sets $\hat{\mathcal{C}} = \{\bm{c}^{(1)}, \dots, \bm{c}^{(i)}, \dots, \bm{c}^{(N)}\}$ and $\hat{\mathcal{C}}' = \{\bm{c}^{(1)}, \dots, \Tilde{\bm{c}}^{(i)}, \dots, \bm{c}^{(N)}\}$ that differ by exactly one sample at index $i$.
The difference in the empirical cost is at most $2B / N$:
\begin{equation}
\begin{aligned}
    &\left| \hat{\mathcal{J}}_N(\bm{\theta}) - \hat{\mathcal{J}}_{N'}(\bm{\theta})\right| \\
    =& \left| \frac{1}{N}\sum_{j=1}^N \mathcal{L}\left(\bm{\theta}, \bm{c}^{(j)}\right) - \frac{1}{N}\left( \mathcal{L}\left(\bm{\theta}, \Tilde{\bm{c}}^{(i)}\right) + \sum_{j \neq i} \mathcal{L}\left(\bm{\theta}, \bm{c}^{(j)} \right) \right) \right| \\
    =& \frac{1}{N} \left| \mathcal{L}\left(\bm{\theta}, \bm{c}^{(i)} \right) - \mathcal{L}\left(\bm{\theta}, \Tilde{\bm{c}}^{(i)}\right) \right|\\
    \leq &\frac{1}{N} \left( \left| \mathcal{L}\left(\bm{\theta}, \bm{c}^{(i)} \right) \right| + \left| \mathcal{L}\left(\bm{\theta}, \Tilde{\bm{c}}^{(i)} \right) \right| \right) = \frac{2B}{N} 
\end{aligned}
\end{equation}
Because the supremum is a contraction, 
the function $\Phi(\hat{\mathcal{C}})$ satisfies the bounded difference condition (i.e., substituting the value of $\bm{c}^{(i)}$ changes the value of $\Phi(\hat{\mathcal{C}})$ by at most $2B/N$):
\begin{equation}
    \left|\Phi(\hat{\mathcal{C}}) - \Phi(\hat{\mathcal{C}}')\right| \le \sup_{\bm{\theta}} \left|\hat{\mathcal{J}}_N(\bm{\theta}) - \hat{\mathcal{J}}_{N'}(\bm{\theta}) \right| \le \frac{2B}{N}.
\end{equation}
We now apply McDiarmid's Inequality:
\begin{equation}
\begin{aligned}
    \mathbb{P}\left(\Phi(\hat{\mathcal{C}}\right) - \mathbb{E}[\Phi(\hat{\mathcal{C}})] \ge \epsilon)
    & \le \exp\left( \frac{-2\epsilon^2}{N (2B/N)^2} \right) \\
    & = \exp\left( \frac{-N\epsilon^2}{2B^2} \right)
\end{aligned}
\end{equation}
Setting the right-hand side to $\delta$ and solving for $\epsilon$, we obtain the concentration bound with:
\begin{equation}
    \epsilon = 2B\sqrt{\frac{\ln(1/\delta)}{2N}}.
\end{equation}
Finally, we bound the expectation $\mathbb{E}[\Phi(\hat{\mathcal{C}})]$ via standard symmetrization. 
Let $\hat{\mathcal{C}}_g$ be an independent ghost sample and $\sigma_i$ i.i.d. Rademacher variables. Then:
\begin{equation}
\begin{aligned}
    \mathbb{E}_{\hat{\mathcal{C}}}[\Phi(\hat{\mathcal{C}})] &= \mathbb{E}_{\hat{\mathcal{C}}} \left[ \sup_{\bm{\theta}} \mathbb{E}_{\hat{\mathcal{C}}_g} \left[\hat{\mathcal{J}}_{N_g}(\bm{\theta}) - \hat{\mathcal{J}}_N(\bm{\theta}) \right] \right] \\
    &\le \mathbb{E}_{\hat{\mathcal{C}}, \hat{\mathcal{C}}_g} \left[ \sup_{\bm{\theta}} \frac{1}{N} \sum_{i=1}^N \left(\mathcal{L}(\bm{\theta}, \bm{c}_g^{(i)}) - \mathcal{L}(\bm{\theta}, \bm{c}^{(i)}) \right) \right] \\
    &\le 2 \mathbb{E}_{\hat{\mathcal{C}}, \sigma} \left[ \sup_{\bm{\theta}} \frac{1}{N} \sum_{i=1}^N \sigma_i \, \mathcal{L}(\bm{\theta}, \bm{c}^{(i)}) \right] = 2\mathfrak{R}_N\left(\mathcal{L}_{\bm{\Theta}}\right)
\end{aligned}
\end{equation}
Combining the concentration and expectation bounds yields the result for Step 1:
\begin{equation}
\begin{aligned}
    \left|\mathcal{J}_S(\bm{\theta}) - \hat{\mathcal{J}}_N(\bm{\theta}) \right| &\leq \sup_{\bm{\theta} \in \bm{\Theta} } \left|\mathcal{J}_S(\bm{\theta}) - \hat{\mathcal{J}}_N(\bm{\theta}) \right| \\
    & \leq 2\mathfrak{R}_N\left(\mathcal{L}_{\bm{\Theta}}\right) + 2B\sqrt{\frac{\ln(1/\delta)}{2N}}
\end{aligned}
\end{equation}

\textbf{Step 2: Bounding the Transfer Error (Physics Mismatch).}
We must bound the discrepancy $\left|\mathcal{J}_T(\bm{\theta}) - \mathcal{J}_S(\bm{\theta})\right|$ caused by the shift from $p_S(\cdot)$ to $p_T(\cdot)$ by:
\begin{equation}
\begin{aligned}
     &\left| \mathcal{J}_T(\bm{\theta}) - \mathcal{J}_S (\bm{\theta}) \right| \\
\leq &\sup_{\bm{\theta} \in \bm{\Theta}}  \left| \mathcal{J}_T(\bm{\theta}) - \mathcal{J}_S (\bm{\theta}) \right| \\
    =&  \sup_{\bm{\theta} \in \bm{\Theta}} \left| \mathbb{E}_{\bm{c}\sim p_T(\bm{c})}[\mathcal{L}(\bm{\theta}, \bm{c})] - \mathbb{E}_{\bm{c}\sim p_S(\bm{c})}[\mathcal{L}(\bm{\theta}, \bm{c})] \right| \\
    =& \sup_{\bm{\theta} \in \bm{\Theta}}  \left| \int_{\mathcal{C}} \mathcal{L}(\bm{\theta}, \bm{c})\, p_T(\bm{c})\, d\bm{c} - \int_{\mathcal{C}} \mathcal{L}(\bm{\theta}, \bm{c})\,p_S(\bm{c})\, d\bm{c} \right|\\
    =& d_{\mathcal{L}_{\bm{\Theta}}} (p_S, p_T)
\end{aligned}
\end{equation}
where $d_{\mathcal{L}_{\bm{\Theta}}} (p_S, p_T)$ is the IPM defined in Eq.~\eqref{eq:ipm}, measuring the worst-case performance gap over the hypothesis (policy parameter) space $\bm{\Theta}$.
Unlike the estimation error, this does not vanish with $N \to \infty$. 

\textbf{Conclusion.}
Substituting the bounds from Step 1 and Step 2 back into the decomposition, we obtain:
\begin{equation}
    \mathcal{J}_T(\bm{\theta}) \le \hat{\mathcal{J}}_N(\bm{\theta}) + \left( 2\mathfrak{R}_N + 2B\sqrt{\frac{\ln(1/\delta)}{2N}} \right) + d_{\mathcal{L}_{\bm{\Theta}}}(p_S, p_T).
\end{equation}
\end{proof}

\textbf{Remark:}
Note that if the simulator cannot model the real world (e.g., $p_T$ is disjoint from $p_S$), 
the discrepancy term $d_{\mathcal{L}_{\bm{\Theta}}}$ is large and irreducible with $N$. 
However, compared to conventional DR ($N$=1), DRIS ($N \gg 1$) minimizes the first term (Generalization Gap), ensuring that the policy is at least robust to the \textit{modeled} uncertainty. 

\subsection{Discussion on Theory and Practical Implementation}
\label{appendix:f}

The theoretical analysis provided in this appendix examines the properties of the DRIS framework within an idealized setting, assuming access to the full ensemble state.
In our practical implementation, however, the policy operates on learned, compressed set representations generated by a trained encoder. 
While a formal equivalence is not proved here, 
the encoder is explicitly optimized to preserve task-relevant information from the underlying ball states. 
Consequently, we anticipate that the properties derived in this analysis remain positively correlated with the system's practical performance.
Extending this theoretical framework to encompass a broader range of learned representations is deferred to future work.